\documentclass{article}
\ifdefined\pdfoutput\pdfoutput=1\fi
\usepackage{arxiv}
\usepackage{iftex}
\ifPDFTeX
  \usepackage[utf8]{inputenc}
  \usepackage[T1]{fontenc}
\else
  \usepackage{fontspec}
\fi
\usepackage{amsmath,amssymb,amsthm}
\usepackage{stmaryrd}
\usepackage{mathtools}
\usepackage{enumitem}
\usepackage{booktabs}
\usepackage{hyperref}
\hypersetup{
	hidelinks,
	linktoc = all,
	pdfdisplaydoctitle,
	breaklinks,
	pdfstartview = Fit,
	unicode,
	pdftitle={Exact semantic readout from compressed vector representations},
	pdfsubject={vector logic, formal semantics, intensional semantics},
	pdfauthor={Daniel Quigley},
	pdfcreator = {LuaLaTeX},
	pdfkeywords={formal semantics, vector logic, semantic space, encoding, embedding, mathematical linguistics},
}
\usepackage[activate={true,nocompatibility},final]{microtype}
\usepackage{doi}
\usepackage{array}
\usepackage{cleveref}
\usepackage[most]{tcolorbox}
\usepackage{tikz}
\usetikzlibrary{cd,arrows.meta,positioning,calc,decorations.pathreplacing}
\newcommand{\R}{\mathbb{R}}
\newcommand{\IM}{\operatorname{im}}
\newcommand{\rs}{\operatorname{row}}
\newcommand{\cs}{\operatorname{col}}
\newcommand{\rk}{\operatorname{rank}}
\newcommand{\Hom}{\operatorname{Hom}}
\newcommand{\vecop}{\operatorname{vec}}
\newcommand{\dg}{\operatorname{deg}}
\newcommand{\bA}{\mathbf{A}}
\newcommand{\bH}{\mathbf{H}}
\newcommand{\bE}{\mathbf{E}}
\newcommand{\bM}{\mathbf{M}}
\newcommand{\bP}{\mathbf{P}}
\newcommand{\bT}{\mathbf{T}}
\newcommand{\bX}{\mathbf{X}}
\newcommand{\bY}{\mathbf{Y}}
\newcommand{\bG}{\mathbf{G}}
\newcommand{\bW}{\mathbf{W}}
\newcommand{\bb}{\mathbf{b}}
\newcommand{\bv}{\mathbf{v}}
\newcommand{\be}{\mathbf{e}}
\newcommand{\bt}{\mathbf{t}}
\newcommand{\bw}{\mathbf{w}}
\newcommand{\bc}{\mathbf{c}}
\newcommand{\ones}{\mathbf{1}}
\newcommand{\Loss}{\mathcal{L}}
\newcommand{\Lex}{\mathcal{P}}
\newcommand{\den}[1]{\llbracket #1 \rrbracket}
\newcommand{\dan}{\textsc{Daniel}}
\newcommand{\tom}{\textsc{Thomas}}

\definecolor{boxcolor1}{RGB}{255,254,230}
\definecolor{boxcolor2}{RGB}{252,227,215}
\definecolor{boxcolor3}{RGB}{225,240,227}
\definecolor{boxcolor4}{RGB}{227,235,246}

\tcbset{
	envbox/.style={
		breakable,
		before skip=1.5\topskip,
		after skip=1.5\topskip,
		left skip=0pt,
		right skip=0pt,
		left=4pt,
		right=4pt,
		top=2pt,
		bottom=2pt,
		lefttitle=4pt,
		righttitle=4pt,
		toptitle=2pt,
		bottomtitle=2pt,
		sharp corners,
		boxrule=0pt,
		titlerule=.4pt,
		coltitle=black,
		colframe=darkgray,
		coltext=black,
		fonttitle=\bfseries,
	}
}
\newtcolorbox{definitionbox}[1][]{envbox, colback=boxcolor1, colbacktitle=boxcolor1, #1}
\newtcolorbox{theorembox}[1][]{envbox, colback=boxcolor2, colbacktitle=boxcolor2, #1}
\newtcolorbox{lemmabox}[1][]{envbox, colback=boxcolor3, colbacktitle=boxcolor3, #1}
\newtcolorbox{propositionbox}[1][]{envbox, colback=boxcolor3, colbacktitle=boxcolor3, #1}
\newtcolorbox{corollarybox}[1][]{envbox, colback=boxcolor4, colbacktitle=boxcolor4, #1}

\newtheoremstyle{mystyle}{3pt}{3pt}{\itshape}{}{\bfseries}{.}{.5em}{}
\theoremstyle{mystyle}
\newtheorem{example}{Example}[section]
\newtheorem{remark}{Remark}[section]

\newtheorem{definition}{Definition}[section]
\renewenvironment{definition}[1][]{%
	\refstepcounter{definition}%
	\begin{definitionbox}[title=Definition~\thedefinition\ifx\relax#1\relax\else\ (#1)\fi]%
	}{%
	\end{definitionbox}
}
\newtheorem{theorem}{Theorem}[section]
\renewenvironment{theorem}[1][]{%
	\refstepcounter{theorem}%
	\begin{theorembox}[title=Theorem~\thetheorem\ifx\relax#1\relax\else\ (#1)\fi]%
	}{%
	\end{theorembox}
}
\newtheorem{lemma}{Lemma}[section]
\renewenvironment{lemma}[1][]{%
	\refstepcounter{lemma}%
	\begin{lemmabox}[title=Lemma~\thelemma\ifx\relax#1\relax\else\ (#1)\fi]%
	}{%
	\end{lemmabox}
}
\newtheorem{proposition}{Proposition}[section]
\renewenvironment{proposition}[1][]{%
	\refstepcounter{proposition}%
	\begin{propositionbox}[title=Proposition~\theproposition\ifx\relax#1\relax\else\ (#1)\fi]%
	}{%
	\end{propositionbox}
}
\newtheorem{corollary}{Corollary}[section]
\renewenvironment{corollary}[1][]{%
	\refstepcounter{corollary}%
	\begin{corollarybox}[title=Corollary~\thecorollary\ifx\relax#1\relax\else\ (#1)\fi]%
	}{%
	\end{corollarybox}
}

\title{Exact semantic readout from compressed vector representations\thanks{Third in a series developing a vector logic for formal semantics \cite{Quigley2025,quigley2026intensional}.}}

\author{ \href{https://orcid.org/0009-0004-7957-1806}{\includegraphics[scale=0.06]{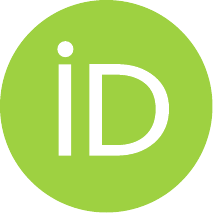}\hspace{1mm}Daniel Quigley} \\
	Center for Possible Minds\\
	Indiana University Bloomington\\
	Bloomington, IN 47408 \\
	\texttt{dgquigle@iu.edu} \\
}

\renewcommand{\headeright}{Preprint}
\renewcommand{\undertitle}{Preprint}
\renewcommand{\shorttitle}{Exact semantic readout in a vector logic}

\begin{document}
\maketitle

\begin{abstract}
    We characterize when compressed vector representations admit exact linear or affine readouts of a finite lexicon's truth conditions: one fixed map per predicate, sending each entity vector to the corresponding truth vector. A necessary and sufficient row-space condition determines existence; the augmented truth matrix has rank $r$, giving minimum dimension $r$ in the linear case, and $r-1$ in the affine. Exact readouts return values in a shared truth basis on which Boolean connectives act unchanged; separability alone requires an intervening threshold. For binary relations, exact bilinear readout of identity or strict total order requires linearly independent entity vectors. Experiments with GloVe and word2vec distinguish exact affine recovery, linear separability, and held-out prediction: most predicates are strictly separable, but none admits an exact affine readout from the pretrained embeddings. Supervised transductive training attains exact affine recovery to numerical precision at every tested dimension meeting the bound. At the embeddings' original dimension, geometries constrained to exact linear recovery retain 98--99 percent of the pretrained variance on the feature norms, and 80--83 percent on the WordNet lexicon.

\end{abstract}


\section{Introduction}\label{sec:introduction}

Montagovian semantics represents entities and predicates in typed domains, with composition governed by the model’s functions \cite{Montague1974,Heim1998}. Distributional embeddings represent lexical items as vectors, whose geometry reflects patterns of use \cite{Boleda2020}. When these vectors represent the entities of a finite semantic model, which predicates can be recovered by linear readouts, and what does compression prevent?

A vector logic for formal semantics \cite{Quigley2025,quigley2026intensional} supplies an exact construction: each primitive domain element receives its own basis vector, and semantic functions extend linearly from free carriers, preserving composition along primitive intermediate types. Linear independence makes these lifts possible. Learned embeddings generally have fewer dimensions than entities, and, therefore, introduce linear dependences that may obstruct the lifts.

An embedding lookup with matrix $\bE \in \R^{V \times d}$ is equivalent, on one-hot inputs, to a bias-free linear map \cite{raschka-faq}. It, therefore, factors through the free entity carrier as $\bE^\top : \R^V \to \R^d$. This identifies two regimes: a \emph{free} geometry has linearly independent entity vectors, and admits every semantic lift; a \emph{compressed} geometry has dependent vectors, and admits only those lifts compatible with its dependences. Compression need not erase entity identity: distinct columns still permit arbitrary predicate lookup on a finite domain; we are interested, then, in which predicates remain linearly recoverable.

For a monadic lexicon with truth matrix $\bT$, we show that exact readouts into a shared truth basis exist precisely when the geometry’s row space contains $\rs[\bT;\ones]$. The rank of this augmented truth matrix gives the minimum dimension, while its kernel specifies the admissible dependences among entity vectors. A restricted lexicon can, therefore, admit exact compression, although requiring every predicate forces the free regime. Once atomic readouts return the truth basis, Boolean connectives act unchanged.

Exact readout is stronger than linear separability: a score may distinguish true instances from false ones without, itself, returning their truth values. This distinction connects linear probing \cite{belinkov2022probing} to the vector logic. Whether such a readout exists, and whether it can be learned from a subset of entities, are separate questions. We test exact recovery, separability, and held-out prediction on distributional embeddings, then examine how much pretrained variance survives supervised training toward exact recovery. Figure~\ref{fig:architecture} summarizes the program and the dependencies among these questions.

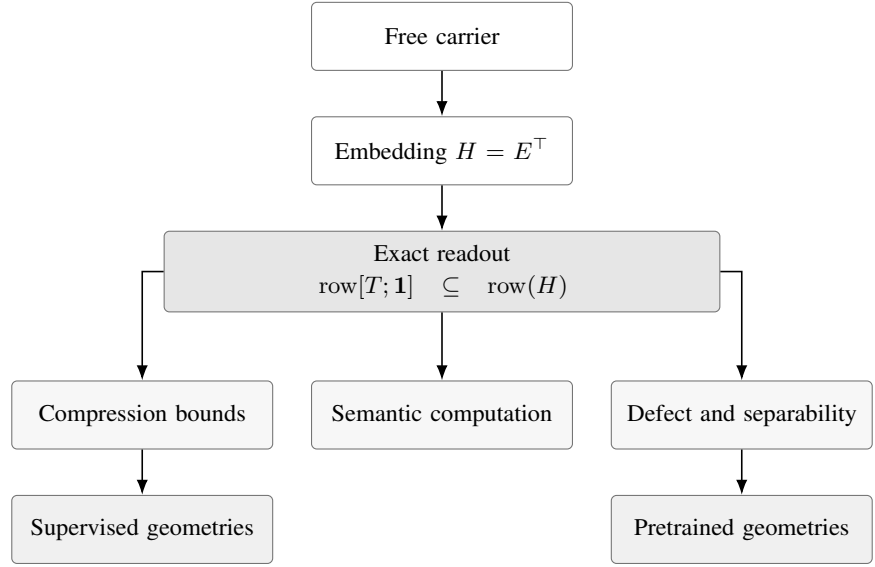
\begin{figure}[ht]
\centering
\begin{tikzpicture}[
  font=\small,
  node distance=6mm and 5mm,
  box/.style={
    draw=black!55,
    rounded corners=2pt,
    align=center,
    inner sep=5pt,
    minimum height=9mm,
    text width=3.1cm
  },
  core/.style={
    box,
    text width=7cm,
    fill=black!10
  },
  result/.style={box, fill=black!3},
  empirical/.style={box, fill=black!6},
  arrow/.style={-{Latex[length=2mm]}, semithick}
]

\node[box] (free) {Free carrier};

\node[box, below=of free] (geometry) {
  Embedding $H=E^\top$
};

\node[core, below=of geometry] (criterion) {
  Exact readout\\[2pt]
  $\operatorname{row}[T;\mathbf1]
    \subseteq \operatorname{row}(H)$
};

\node[result, below=9mm of criterion] (composition) {
  Semantic computation
};

\node[result, left=of composition] (structure) {
  Compression bounds
};

\node[result, right=of composition] (relaxation) {
  Defect and separability
};

\node[empirical, below=of structure] (training) {
  Supervised geometries
};

\node[empirical, below=of relaxation] (measurement) {
  Pretrained geometries
};

\draw[arrow] (free) -- (geometry);
\draw[arrow] (geometry) -- (criterion);

\draw[arrow] (criterion) -- (composition);
\draw[arrow] (criterion.west) -| (structure.north);
\draw[arrow] (criterion.east) -| (relaxation.north);

\draw[arrow] (structure) -- (training);
\draw[arrow] (relaxation) -- (measurement);

\end{tikzpicture}
\caption{Dependencies among the formal results and empirical analyses.}
\label{fig:architecture}
\end{figure}

\section{Vector logic}\label{sec:vl}

We recall from \cite{Quigley2025} the definitions and the homomorphism theorem, and from \cite{quigley2026intensional} the index-sort material. We refrain from re-proving the relevant content; see, instead, papers proper therein.

\subsection{Extensional models and free carrier}\label{subsec:extensional}

Types are generated from $e$ and $t$ by $\langle \sigma, \tau \rangle$. A typed extensional model $\mathcal{M}_{ext} = \langle (\mathcal{D}_\tau)_\tau, \mathcal{I} \rangle$ has entity domain $\mathcal{D}_e$, truth domain $\mathcal{D}_t = \{1, 0\}$, function domains $\mathcal{D}_{\langle \sigma, \tau \rangle} = \mathcal{D}_\tau^{\mathcal{D}_\sigma}$, and an interpretation $\mathcal{I}$; denotations $\den{\cdot}$ follow the recursion of \cite{Heim1998}. Throughout, let $\mathcal{D}_e$ be finite\footnote{The finiteness restriction is what makes the rank statements below meaningful; the extensional theorem itself holds for domains of any cardinality.}, with $|\mathcal{D}_e| = V$ and elements $d_1, \dots, d_V$.

The vector space model $\mathcal{M}_{\mathcal{S}}$ assigns to each domain a real vector space $\mathcal{S}_{\mathcal{D}_\tau}$ and an injection $h_\tau : \mathcal{D}_\tau \to \mathcal{S}_{\mathcal{D}_\tau}$. The construction takes the form
\[
  h_e(d_i) = \be_i \in \R^V, \qquad h_t(1) = \bb_1 = \begin{bmatrix}1\\0\end{bmatrix}, \qquad h_t(0) = \bb_0 = \begin{bmatrix}0\\1\end{bmatrix},
\]
where $\be_i$ is the $i$-th standard basis vector, and for a function type sends $f \in \mathcal{D}_{\langle \sigma, \tau \rangle}$ to the linear map $L_f$, determined on the basis $\{\be_a\}_{a \in \mathcal{D}_\sigma}$ of the free carrier of $\sigma$ by $L_f\, \be_a = h_\tau(f(a))$; a linear map is determined by its values on a basis, so this fixes $L_f$ uniquely. Note that \cite{Quigley2025} embeds every domain, function-type domains included, by basis vectors, and defines the lift $h_f$ pointwise on the image of $h_\sigma$ through the left inverse $h_\sigma^{-1}$; \cite{quigley2026intensional} attaches to each type a free carrier $\mathcal{F}_\tau$, with basis indexed by $\mathcal{D}_\tau$, and an operator carrier $\mathcal{S}_\tau$, with $\mathcal{S}_{\langle \sigma, \tau \rangle} = \Hom(\mathcal{F}_\sigma, \mathcal{S}_\tau)$, so that the free carrier stands in argument position, and the two carriers coincide on primitive types. We adopt the second convention, since the results below concern linear readouts on the entity space, whose type is primitive, and we call the family $\{\mathcal{F}_\tau\}$ the free carrier: elements of primitive domains go to basis vectors, and the primitive spaces are free on their domains. In particular, a monadic predicate $P \in \mathcal{D}_{\langle e, t \rangle}$ becomes the $2 \times V$ matrix
\[
  \bM_P = \bigl[\, h_t(P(d_1)) \;\; \cdots \;\; h_t(P(d_V)) \,\bigr],
\]
whose $i$-th column is the truth vector that $P$ assigns to $d_i$, and $\bM_P \be_i = h_t(P(d_i))$ is functional application. Truth-functional connectives become fixed matrices on tensor powers of the truth space, an $n$-ary connective $c$ acting as a $2 \times 2^n$ matrix $\bM_c$ on $\bigotimes_i h_t(t_i)$; for instance
\[
  \bM_\neg = \begin{bmatrix}0&1\\1&0\end{bmatrix}, \qquad
  \bM_\wedge = \begin{bmatrix}1&0&0&0\\0&1&1&1\end{bmatrix},
\]
in the ordering $\bb_1 \otimes \bb_1, \bb_1 \otimes \bb_0, \bb_0 \otimes \bb_1, \bb_0 \otimes \bb_0$ of the tensor basis, after \cite{Mizraji1992,Westphal2005}.

\subsection{Homomorphism theorem}\label{subsec:homothm}

\begin{theorem}\label{thm:homo}
    For every extensional model $\mathcal{M}_{ext}$, there exist injections $\{h_\tau\}$ into vector spaces $\{\mathcal{S}_{\mathcal{D}_\tau}\}$ such that every semantic function $f : \mathcal{D}_\sigma \to \mathcal{D}_\tau$ has a unique linear lift $L_f : \mathcal{F}_\sigma \to \mathcal{S}_{\mathcal{D}_\tau}$ from the free carrier of $\sigma$, with $L_f(\bb_a) = h_\tau(f(a))$ for every $a \in \mathcal{D}_\sigma$, where $\bb_a$ is the free encoding of $a$, which coincides with $h_\sigma(a)$ when $\sigma$ is primitive; $n$-ary functions lift to multilinear maps on the free carriers, and composition of semantic functions corresponds to composition of lifts along primitive intermediate types.
\end{theorem}

The square
\begin{center}
    \begin{tikzcd}[row sep=large, column sep=large]
      \mathcal{D}_\sigma \arrow[r, "f"] \arrow[d, "h_\sigma"'] & \mathcal{D}_\tau \arrow[d, "h_\tau"] \\
      \mathcal{S}_{\mathcal{D}_\sigma} \arrow[r, "L_f"'] & \mathcal{S}_{\mathcal{D}_\tau}
    \end{tikzcd}
\end{center}
commutes for every $f$ with primitive argument type, which covers every function below, so the recursive evaluation of any expression in $\mathcal{M}_{ext}$ has a step-by-step counterpart in $\mathcal{M}_{\mathcal{S}}$ once arguments are carried in their free encodings. The injections are non-surjective by design: $\IM(h_t) = \{\bb_1, \bb_0\}$ is a proper subset of $\R^2$, and denotation is confined to the images. 

Theorem~\ref{thm:homo} is existential: the free carrier is one family of injections for which the lifts exist, and other families are the subject of Section~\ref{sec:regimes}. The commuting square for a single function is \cite{Quigley2025}, where the lift is defined on the image alone; linearity on the span for primitive argument types, the multilinear lift of $n$-ary functions, and the operator carriers of function types follow \cite{quigley2026intensional}, whose descent theorems determine when a functional of a function domain acts linearly on operator encodings as well: on a power set, exactly the constants, the ultrafilter indicators, and their complements, which on a finite domain are the constants, Montague's individuals $\lambda P.\,P(d)$, and their negations.

\subsection{Index sorts}\label{subsec:indexsorts}

The intensional layer adjoins index sorts (worlds and times, among others), collected in a compound index space $S = \prod_\sigma \mathcal{D}_\sigma$, with its own free carrier $h_S(s) = \be_s$, the carrier $\mathcal{F}_s$ of the compound index type $s$. An intension $g : S \to \mathcal{D}_\tau$ becomes the linear operator $\mathcal{S}_S \to \mathcal{S}_{\mathcal{D}_\tau}$, with $\be_s \mapsto h_\tau(g(s))$; a proposition $\varphi$ becomes $\bP_\varphi \in \R^{2 \times |S|}$, with truth profile $\bv(\varphi) \in \{0,1\}^{|S|}$ its top row. Over a single discrete sort of worlds $W = \{w_1, \dots, w_n\}$ with accessibility matrix $\bA \in \{0,1\}^{n \times n}$, the modal operators of \cite{quigley2026intensional} read
\begin{equation}\label{eq:box}
  (\Box\varphi)(w_i) = 1 \iff (\bA\, (\ones - \bv(\varphi)))_i = 0,
  \qquad
  (\Diamond\varphi)(w_i) = 1 \iff (\bA\, \bv(\varphi))_i > 0,
\end{equation}
a linear accumulation followed by a decision; when the out-degree $\dg_\bA(w_i) = \sum_j \bA_{ij}$ is finite, as it is here, the first condition is equivalent to $(\bA\, \bv(\varphi))_i = \dg_\bA(w_i)$. Further, operators are defined over measure frames, in which the counting measure of the discrete case is but one choice among others; a finite geometry matrix carries the discrete case with finitely many indices, the case to which we restrict ourselves in \Cref{sec:index}.

\section{Embedding lookup on free carrier}\label{sec:lookup}

The observation that an embedding layer is a linear map on one-hot inputs is due to \cite{raschka-faq,raschka-llms}, who also explains why frameworks introduce a transpose and why implementations use the lookup form at all. We develop that here in detail, and apply it to our vector logic.

\subsection{Setup}\label{subsec:setup}

Let $V \ge 1$ be the vocabulary size, and $d \ge 1$ the embedding dimension. An embedding matrix is any $\bE \in \R^{V \times d}$, with rows $\bE_i \in \R^{1 \times d}$ and entries $E_{ij}$. The lookup is $\ell_\bE(i) = \bE_i$ for $i \in \{1, \dots, V\}$, and the bias-free linear layer with weight $\bE$ is $L_\bE(x) = x^{\top} \bE$ for $x \in \R^V$. Outputs are row vectors, such that batches stack vertically. Under the reading of the vocabulary as the entity domain, $\be_i = h_e(d_i)$ and $L_\bE(\be_i) = (\bE^{\top} \be_i)^{\top}$; the composite $\bE^{\top} \circ h_e$ is a second injection of $\mathcal{D}_e$ into a vector space, whenever the rows of $\bE$ are distinct.

\begin{remark}
    The vocabulary of a language model and the entity domain of a model are different objects; the identification $\be_i = h_e(d_i)$ is the vocabulary as a set of names, one per entity. Everything below concerns the geometry matrix of Section~\ref{sec:regimes}, and holds for whatever the columns index; the vocabulary reading is the instance in which the geometry is a trained embedding matrix.
\end{remark}

\subsection{Forward pass}\label{subsec:forward}

\begin{proposition}[single token]\label{prop:single}
    For every $i \in \{1, \dots, V\}$, $L_\bE(\be_i) = \ell_\bE(i)$.
\end{proposition}

\begin{proof}\label{prf:single}
    Fix a column $j$. By the definition of the matrix product,
    \[
      (\be_i^{\top} \bE)_j = \sum_{k=1}^{V} (\be_i)_k E_{kj} = \sum_{k=1}^{V} \delta_{ik} E_{kj} = E_{ij},
    \]
    the $j$th entry of $\bE_i$. Since $j$ was arbitrary, $\be_i^{\top} \bE = \bE_i$.
\end{proof}

\begin{corollary}[batch]\label{cor:batch}
    Let $i_1, \dots, i_n \in \{1, \dots, V\}$ and let $\bX \in \R^{n \times V}$ have rows $\be_{i_1}^{\top}, \dots, \be_{i_n}^{\top}$. Then row $t$ of $\bX \bE$ equals $\bE_{i_t}$ for every $t$, so $\bX \bE$ is the vertical stack of $\ell_\bE(i_1), \dots, \ell_\bE(i_n)$.
\end{corollary}

\begin{proof}\label{prf:batch}
    Row $t$ of $\bX \bE$ is $\be_{i_t}^{\top} \bE = \bE_{i_t}$ by Proposition~\ref{prop:single}, applied to each row independently; repeated identifiers among $i_1, \dots, i_n$ are, likewise, covered.
\end{proof}

\begin{remark}
    We should note that, for implementation's sake, frameworks orient the weight of a linear layer differently. In PyTorch, for example, a bias-free linear layer, with input dimension $V$ and output dimension $d$, stores $\bW \in \R^{d \times V}$, and computes $x \mapsto x \bW^{\top}$ \cite{pytorch-linear}; \cite{raschka-faq} identifies this transpose as the usual source of confusion in reproducing the identity in code.
\end{remark}

\begin{corollary}[orientation]\label{cor:orient}
    With $\bW = \bE^{\top}$, the map $x \mapsto x \bW^{\top}$ agrees with $L_\bE$ on $\R^V$ and, hence, with $\ell_\bE$ on one-hot inputs.
\end{corollary}

\begin{proof}\label{prf:orient}
    $\bW^{\top} = \bE$, so $x \bW^{\top} = x \bE$ for every row vector $x$; apply Corollary~\ref{cor:batch}.
\end{proof}

$\bW = \bE^{\top}$ is the geometry matrix: its columns are the entity vectors $\bE_i^{\top}$, and the stored parameter of the framework's linear layer is the geometry itself.

\subsection{Backward pass}\label{subsec:backward}

When considering the backward pass, we are, essentially, encountering gradients. The identity extends to gradients with respect to the parameters, which makes the implementations interchangeable during training, and covers accumulation on repeated identifiers. Let $\Loss$ be a scalar loss depending on the parameters only by $\bY = \bX \bE \in \R^{n \times d}$, and write $\bG = \partial \Loss / \partial \bY \in \R^{n \times d}$ with rows $\bG_1, \dots, \bG_n$. In the lookup, the backward pass is defined as a scatter-add: the gradient with respect to $\bE$ is initialized to zero, and $\bG_t$ is added to row $i_t$ for each $t$.

\begin{proposition}[gradient]\label{prop:grad}
    Under the linear-layer parameterization, $\partial \Loss / \partial \bE = \bX^{\top} \bG$, whose row $k$ equals $\sum_{t :\, i_t = k} \bG_t$, with the empty sum being zero. 
\end{proposition}

Proposition~\ref{prop:grad} coincides with the scatter-add gradient of the lookup implementation.

\begin{proof}\label{prf:grad}
    Since $Y_{tj} = \sum_k X_{tk} E_{kj}$, the chain rule gives
    \[
      \frac{\partial \Loss}{\partial E_{kj}}
      = \sum_{t=1}^{n} \frac{\partial \Loss}{\partial Y_{tj}} \frac{\partial Y_{tj}}{\partial E_{kj}}
      = \sum_{t=1}^{n} G_{tj} X_{tk}
      = (\bX^{\top} \bG)_{kj}.
    \]
    Because $X_{tk} = \delta_{i_t k}$, the sum over $t$ retains exactly the indices with $i_t = k$, so row $k$ of $\bX^{\top} \bG$ is $\sum_{t : i_t = k} \bG_t$. The scatter-add produces the same row by construction: it deposits $\bG_t$ into row $i_t$ for each $t$, and leaves the remaining rows at zero.
\end{proof}

Under Corollary~\ref{cor:orient}, the gradient with respect to $\bW = \bE^{\top}$ is $(\bX^{\top} \bG)^{\top} = \bG^{\top} \bX$; the identification is a transpose, and the row-selection structure is unchanged. Training moves only the columns of the geometry, indexed by tokens observed in the batch, so any property of the trained geometry is a property of the training distribution and the objective. The linear-layer backward materializes $\bX^{\top} \bG \in \R^{V \times d}$ densely, with, at most, $n$ nonzero rows, whereas the scatter-add sees only those rows; the free carrier is a mathematical object that implementations never materialize; Table~\ref{tab:cost} shows.

\begin{table}[ht]
    \centering
    \begin{tabular}{lll}
    \toprule
    \textbf{Quantity} & \textbf{One-hot times $\bE$} & \textbf{Lookup} \\
    \midrule
    Input storage & $\Theta(nV)$ & $\Theta(n)$ \\
    Forward cost & $\Theta(nVd)$ & $\Theta(nd)$ \\
    Gradient storage & $\Theta(Vd)$ dense & $\Theta(\min(n, V)\, d)$ sparse \\
    Parameters & $Vd$ & $Vd$ \\
    \bottomrule
    \end{tabular}
    \caption{Cost comparison for a batch of $n$ tokens.}
    \label{tab:cost}
\end{table}

\subsection{Identifier bias and linearity}\label{subsec:biasandlin}

Now, \cite{raschka-faq} states the identity for a bias-free layer. The restriction concerns the identification of parameters, and leaves the function class untouched.

Let $b \in \R^d$, and consider $x \mapsto x^{\top} \bE + b^{\top}$. On the input $\be_i$, this equals $\bE_i + b^{\top}$, which is row $i$ of $\bE' = \bE + \ones b^{\top}$ with $\ones \in \R^V$ the all-ones vector. A linear layer with bias, restricted to one-hot inputs, is, again, a lookup, with table $\bE'$; the set of functions on one-hot inputs realized with bias equals the set realized without, and the parameterization with bias has a $d$-dimensional redundancy, since $(\bE, b)$ and $(\bE + \ones c^{\top}, b - c)$ realize the same function for every $c \in \R^d$. We see, in Section~\ref{sec:affine}, the same redundancy, only from the side of the readout instead, where it becomes the all-ones row of the rank criterion.

$L_\bE$ is linear on $\R^V$ by construction. The composite $i \mapsto \ell_\bE(i)$ on the integers admits a linear extension only when $\bE_i = ic$ for a fixed $c \in \R^{1 \times d}$ and all $i$, which fails for generic $\bE$, as any $\bE$ with $\bE_2 \ne 2 \bE_1$ shows; and the encoding $i \mapsto \be_i$ is, itself, incompatible with integer addition, since $\be_1 + \be_2$ lies outside the set of one-hot vectors. Token identifiers are labels; in the vector logic, this is the statement that the free carrier encodes entities as atoms: every relation among is via the geometry; we see in Section~\ref{sec:rank} which relations a geometry must preserve.

\section{Regimes}\label{sec:regimes}

The extensional theorem is proved for the free carrier. We see in Section~\ref{sec:lookup} that the free carrier is the input to every embedding matrix; any other injection of $\mathcal{D}_e$ into a real vector space is a candidate carrier, and which of them admit the lifts that Theorem~\ref{thm:homo} guarantees for the free one is decided by the geometry matrix of Definition~\ref{def:geom}.

\begin{definition}[geometry matrix]\label{def:geom}
    Let $h : \mathcal{D}_e \to \R^d$ be any map; its geometry matrix is $\bH \in \R^{d \times V}$, with columns $\bH \be_i = h(d_i)$, such that $h = \bH \circ h_e$, where $h_e$ is the free carrier. 
\end{definition}

The free carrier itself has $\bH = I_V$; a trained embedding has $\bH = \bE^{\top}$. We call $h$ free when the columns of $\bH$ are linearly independent, and compressed otherwise.

Every map out of $\mathcal{D}_e$ into a vector space factors through the free carrier in this way, which is the universal property that justifies the name in the first place; the geometry matrix is the linear map through which it factors. Injectivity of $h$ is, technically, a weaker condition than freeness: distinct columns suffice for injectivity; a compressed geometry may well have distinct columns. 

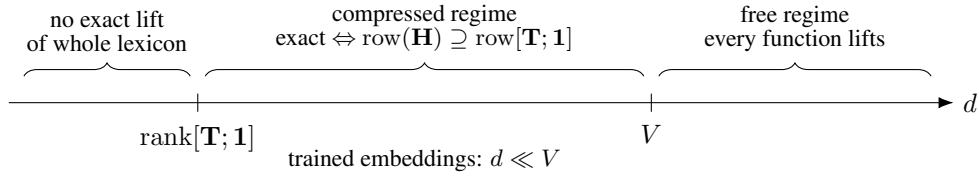
\begin{figure}[h]
    \centering
    \begin{tikzpicture}[x=1cm]
    \draw[-{Latex[length=2mm]}] (0,0) -- (12.5,0) node[right] {$d$};
    \foreach \x/\l in {2.5/{$\rk[\bT;\ones]$}, 8.5/{$V$}} {
      \draw (\x,-0.12) -- (\x,0.12);
      \node[below=2pt] at (\x,-0.1) {\l};
    }
    \draw[decorate, decoration={brace, amplitude=5pt}] (0.2,0.35) -- (2.4,0.35) node[midway, above=6pt, align=center] {\footnotesize no exact lift\\[-2pt]\footnotesize of whole lexicon};
    \draw[decorate, decoration={brace, amplitude=5pt}] (2.6,0.35) -- (8.4,0.35) node[midway, above=6pt, align=center] {\footnotesize compressed regime\\[-2pt]\footnotesize exact $\Leftrightarrow$ $\rs(\bH) \supseteq \rs[\bT;\ones]$};
    \draw[decorate, decoration={brace, amplitude=5pt}] (8.6,0.35) -- (12.2,0.35) node[midway, above=6pt, align=center] {\footnotesize free regime\\[-2pt]\footnotesize every function lifts};
    \node[below=14pt, font=\footnotesize] at (5.5,0) {trained embeddings: $d \ll V$};
    \end{tikzpicture}
    \caption{The dimension axis for a geometry $\bH \in \R^{d \times V}$ over a lexicon with truth matrix $\bT$. Below $\rk[\bT;\ones]$, no geometry carries the whole lexicon; at and above $V$ with independent columns, every semantic function lifts; between them, exact lift is the row-space condition of Theorem~\ref{thm:rank}, and is where the geometries of Section~\ref{sec:measure} are. One instance of each regime is drawn in Figure~\ref{fig:bothregimes}.}
    \label{fig:regimes}
\end{figure}

\subsection{Free regime}\label{subsec:free}

\begin{proposition}[reparameterization]\label{prop:equiv}
    Let $\{h_\tau\}$ be the free carrier with lifts $\{L_f\}$, and, for each type, let $T_\tau$ be an injective linear map on $\mathcal{S}_{\mathcal{D}_\tau}$. Set $h'_\tau = T_\tau \circ h_\tau$. Then, for every semantic function $f : \mathcal{D}_\sigma \to \mathcal{D}_\tau$ with $\sigma$ primitive, there is a linear $L'_f$, with $L'_f \circ h'_\sigma = h'_\tau \circ f$, and composition along primitive types is preserved. 
\end{proposition}

In particular, every free geometry of Definition~\ref{def:geom} admits lifts for every semantic function.

\begin{proof}\label{prf:equiv}
    Since $T_\sigma$ is injective, it has a linear left inverse $T_\sigma^{-}$ with $T_\sigma^{-} T_\sigma = I$ on $\mathcal{S}_{\mathcal{D}_\sigma}$. Define $L'_f = T_\tau L_f T_\sigma^{-}$. For $a \in \mathcal{D}_\sigma$,
    \[
      L'_f\, h'_\sigma(a) = T_\tau L_f T_\sigma^{-} T_\sigma h_\sigma(a) = T_\tau L_f h_\sigma(a) = T_\tau h_\tau(f(a)) = h'_\tau(f(a)).
    \]
    For composition, if $g : \mathcal{D}_\tau \to \mathcal{D}_\rho$ with $\tau$ primitive, then 
    \[L'_g L'_f = T_\rho L_g T_\tau^{-} T_\tau L_f T_\sigma^{-} = T_\rho L_g L_f T_\sigma^{-} = T_\rho L_{g \circ f} T_\sigma^{-} = L'_{g \circ f},\] 
    using Theorem~\ref{thm:homo} for $L_g L_f = L_{g \circ f}$. Finally, a free geometry $\bH \in \R^{d \times V}$ has linearly independent columns, hence is an injective linear map $\R^V \to \R^d$; take $T_e = \bH$ and $T_\tau = I$ for the other types.
\end{proof}

Here, we see that the orthonormal basis of the free carrier is a convenience: any linearly independent family of entity vectors will do. Every function lifts here, so every measurement on a free geometry passes the homomorphism conditions. The function-type spaces are left as they were, so a predicate still lives in $\Hom(\R^V, \R^2)$, while entities live in $\R^d$; the readout picture of Section~\ref{sec:rank}, in which a predicate over a geometry $\bH$ is a linear map $L_P : \R^d \to \R^2$, is the specialization $L_P = \bM_P \bH^{-}$, with $\bH^{-}$ a left inverse of $\bH$.

\subsection{Compressed regime}\label{subsec:compressed}

A trained embedding matrix has $\bE^{\top} \in \R^{d \times V}$ with $d$ in the hundreds (or low thousands), and $V$ in the tens of thousands, so $\rk \bE^{\top} \le d < V$, and the geometry is compressed by a wide margin. Proposition~\ref{prop:equiv} requires injective $T$, and, therefore, leaves it be. A criterion on $\ker \bH$ replaces the proposition, which is the space of linear dependences among the entity vectors: a vector $\bc \in \R^V$ lies in $\ker \bH$ exactly when $\sum_i c_i\, h(d_i) = 0$. We must show that lifts exist exactly when these dependences lie in the kernel of the lexicon's augmented truth matrix.

\section{Rank criterion for monadic predicates}\label{sec:rank}

Fix a geometry $\bH \in \R^{d \times V}$ and a finite lexicon $\Lex \subseteq \mathcal{D}_{\langle e, t \rangle}$ of monadic predicates. For $P \in \Lex$, write $\bt_P \in \{0,1\}^{1 \times V}$ for its truth row, $(\bt_P)_i = P(d_i)$, and let $\bT \in \{0,1\}^{|\Lex| \times V}$ be the truth matrix with rows $\bt_P$. Write $[\bT; \ones]$ for $\bT$ augmented by the all-ones row. Row spaces $\rs(\cdot)$ are subspaces of $\R^{1 \times V}$.

\subsection{Exact lifts}\label{subsec:exactlifts}

An exact lift of $P$ over $\bH$ is a linear map $L_P : \R^d \to \R^2$ with $L_P\, h(d_i) = h_t(P(d_i))$ for every $i$, the truth basis $\bb_1, \bb_0$ being shared across the lexicon, such that the connective matrices of Section~\ref{sec:vl} act on outputs unchanged. In matrix form, $L_P \in \R^{2 \times d}$, and the condition is
\begin{equation}\label{eq:lift}
  L_P\, \bH = \bM_P,
\end{equation}
with $\bM_P$ the predicate matrix of the free carrier. Over the free carrier, $\bH = I_V$ and $L_P = \bM_P$ solves \eqref{eq:lift}; over a compressed geometry, \eqref{eq:lift} is a linear system in the unknown $L_P$ that may well fail to be solvable.

\begin{figure}[h]
    \centering
    \begin{tikzcd}[row sep=large, column sep=huge]
    \mathcal{D}_e \arrow[r, "h_e"] \arrow[d, "P"'] & \R^V \arrow[r, "\bH"] \arrow[d, "\bM_P"] & \R^d \arrow[dl, "L_P", dashed] \\
    \{0,1\} \arrow[r, "h_t"'] & \R^2 &
    \end{tikzcd}
    \caption{The left square commutes for every predicate by Theorem~\ref{thm:homo}, with $\bM_P = \bb_1 \bt_P + \bb_0 (\ones - \bt_P)$ the free readout; the dashed map exists when $\bM_P$ factors through $\bH$, which is condition \eqref{eq:lift}. When $\bH$ is free (injective), the factorization is $L_P = \bM_P \bH^{-}$; when $\bH$ is compressed, it exists only if $\ker \bH \subseteq \ker \bM_P$, which is a rank condition as Theorem~\ref{thm:rank}.}
    \label{fig:lift}
\end{figure}

\begin{lemma}\label{lem:rowspace}
    For every $P$, $\rs(\bM_P) = \operatorname{span}\{\bt_P, \ones\}$.
\end{lemma}

\begin{proof}
    The rows of $\bM_P$ are $\bt_P$ and $\ones - \bt_P$, since column $i$ is $\bb_1$ when $P(d_i) = 1$ and $\bb_0$ otherwise. The spans of $\{\bt_P, \ones - \bt_P\}$ and $\{\bt_P, \ones\}$ coincide.
\end{proof}

\begin{theorem}[rank criterion]\label{thm:rank}
    The following are equivalent.
    \begin{enumerate}
      \item Every $P \in \Lex$ has an exact lift over $\bH$.
      \item $\rs(\bH) \supseteq \rs[\bT; \ones]$.
      \item $\ker \bH \subseteq \ker[\bT; \ones]$; that is, every linear dependence $\sum_i c_i\, h(d_i) = 0$ among the entity vectors satisfies $\sum_i c_i P(d_i) = 0$ for every $P \in \Lex$ and $\sum_i c_i = 0$.
    \end{enumerate}
\end{theorem}

\begin{proof}[Proof sketch]\label{prf:rank}
    The equation $L \bH = \bM$ has a solution $L$ if and only if every row of $\bM$ is a linear combination of the rows of $\bH$, that is, $\rs(\bM) \subseteq \rs(\bH)$. By Lemma~\ref{lem:rowspace}, \eqref{eq:lift} is solvable for $P$ if and only if $\bt_P, \ones \in \rs(\bH)$, and solvability for all of $\Lex$ is $\rs[\bT; \ones] \subseteq \rs(\bH)$, which is the equivalence of (1) and (2). 
    
    For (2) and (3), the row space and kernel of a matrix are orthogonal complements in $\R^V$, so $\rs(\bH) \supseteq \rs[\bT; \ones]$ if and only if $\ker \bH \subseteq \ker [\bT; \ones]$; and $\bc \in \ker[\bT; \ones]$ to $\bt_P \bc = 0$ for each $P$ and $\ones \bc = 0$.
    
    From (2) to (1) also admits a direct construction, and the direction from (1) to (3) a direct computation. Given (2), choose row vectors $\bw_P, \bw'_P \in \R^{1 \times d}$ with $\bw_P \bH = \bt_P$ and $\bw'_P \bH = \ones - \bt_P$, and set $L_P = \bb_1 \bw_P + \bb_0 \bw'_P$, so that $L_P x = (\bw_P x)\, \bb_1 + (\bw'_P x)\, \bb_0$. Then $L_P h(d_i) = P(d_i)\, \bb_1 + (1 - P(d_i))\, \bb_0 = h_t(P(d_i))$. Conversely, given a lift $L_P$ and $\bc \in \ker \bH$, apply $L_P$ to $\sum_i c_i\, h(d_i) = 0$ and expand: $\sum_i c_i \bigl(P(d_i)\, \bb_1 + (1 - P(d_i))\, \bb_0\bigr) = 0$, and independence of $\bb_1, \bb_0$ forces $\sum_i c_i P(d_i) = 0$ and $\sum_i c_i (1 - P(d_i)) = 0$, whose sum is $\sum_i c_i = 0$.
\end{proof}

\subsection{Corollaries}\label{subsec:corollaries}

Here we take a brief tour of some nice properties of the lifts.

\begin{corollary}[minimal dimension]\label{cor:mindim}
    The least $d$ for which some $\bH \in \R^{d \times V}$ carries exact lifts of all of $\Lex$ is $\rk[\bT; \ones]$, attained by any $\bH$, whose rows form a basis of $\rs[\bT; \ones]$.
\end{corollary}

\begin{proof}[Proof sketch]\label{prf:mindim}
    By (2) in Theorem~\ref{thm:rank}, $\rs(\bH)$ must contain a subspace of dimension $\rk[\bT; \ones]$, so $d \ge \rk \bH \ge \rk[\bT; \ones]$; taking the rows of $\bH$ to be a basis of $\rs[\bT; \ones]$ gives equality.
\end{proof}

\begin{corollary}[closure forces the free regime]\label{cor:closure}
    If $\Lex = \mathcal{D}_{\langle e, t \rangle}$, then exact lifts of all of $\Lex$ exist only over free geometries; in particular, $d \ge V$.
\end{corollary}

\begin{proof}[Proof sketch]\label{prf:closure}
    $\mathcal{D}_{\langle e, t \rangle}$ contains the singleton predicates $P_i$ with $\bt_{P_i} = \be_i^{\top}$, so $\rs[\bT; \ones] = \R^{1 \times V}$, and (2) in \Cref{thm:rank} forces $\rk \bH = V$, which is linear independence of the $V$ columns.
\end{proof}

The extensional theorem embeds every domain $\mathcal{D}_{\langle e, t \rangle}$ in full; Corollary~\ref{cor:closure} requires a free geometry for it. Linear independence of the entity vectors is, thereby, derived from predicate closure, and the one-hot construction is one coordinate choice among the free geometries, which are its injective linear images by Definition~\ref{def:geom}, each admitting every lift by Proposition~\ref{prop:equiv}; the equal pairwise distances of the one-hot basis are a feature of that choice alone. The theorem embeds an object, the full function space $\mathcal{D}_{\langle e, t \rangle}$, that lies beyond what any learned system represents; the results that follow concern the compressed regime, where closure fails by construction, and the conditions have content.

\begin{corollary}[compressibility and the lexicon]\label{cor:dual}
    A dependence $\bc \in \R^V$ among entity vectors is admissible, in the sense that some geometry with $\bc \in \ker \bH$ carries exact lifts of $\Lex$, if and only if $\bc \perp \rs[\bT; \ones]$; the admissible dependences form the orthogonal complement of the augmented row space of the lexicon, of dimension $V - \rk[\bT; \ones]$.
\end{corollary}

\begin{proof}\label{prf:dual}
    Immediate from (3) in Theorem~\ref{thm:rank}, taking $\bH$ with row space exactly $\rs[\bT; \ones]$ for the converse.
\end{proof}

A geometry may well identify entity vectors up to a dependence exactly when every predicate in the lexicon, and the constant predicate, assign that dependence weight zero. The dimension $d$ of a learned embedding bounds its rank; whether $d$ is compatible with exactness depends on the lexicon, which lexicalizes a minuscule fraction of the $2^V$ available predicates, and Corollary~\ref{cor:mindim} computes the exchange rate between rank and lexicon.

\begin{corollary}[indiscernibility]\label{cor:leibniz}
    Let $\bH$ have $\rs(\bH) = \rs[\bT; \ones]$. Then $h(d_i) = h(d_j)$ if and only if $P(d_i) = P(d_j)$ for every $P \in \Lex$. Hence, a minimal exact geometry is injective on $\mathcal{D}_e$ if and only if $\Lex$ separates entities.
\end{corollary}

\begin{proof}[Proof sketch]\label{prf:leibniz}
    $h(d_i) = h(d_j)$ if and only if $\be_i - \be_j \in \ker \bH = \rs[\bT; \ones]^{\perp}$, if and only if $\bt_P(\be_i - \be_j) = 0$ for all $P$ (the condition $\ones(\be_i - \be_j) = 0$ holding automatically), which is $P(d_i) = P(d_j)$ for all $P$.
\end{proof}

\begin{remark}[canonical geometry]\label{rem:canonical}
    The operator carrier of the predicate type is, itself, a compressed geometry. Let the $2^V$ predicates of $\mathcal{D}_{\langle e, t \rangle}$ play the role of entities, and Montague's $V$ individuals $\lambda P.\,P(d_i)$ the role of a lexicon, with truth rows $\bt_{d_i}(P) = P(d_i)$. The comparison map $\varepsilon : \R^{2^V} \to \Hom(\R^V, \R^2)$ of \cite{quigley2026intensional}, which sends the basis vector of $P$ to $\bM_P$, has, as its matrix, the $2V \times 2^V$ array, whose rows are $\bt_{d_i}$ and $\ones - \bt_{d_i}$, so $\rs(\varepsilon) = \rs[\bT_{\mathrm{ind}}; \ones]$, of rank $V + 1$. By Corollary~\ref{cor:mindim}, $\varepsilon$ is a minimal exact geometry for the lexicon of individuals, and, by Theorem~\ref{thm:rank}, a Boolean-valued functional of predicates lifts over it when its truth row lies in the affine span of the coordinate functions on the cube $\{0,1\}^V$, which, for Boolean values, means a constant, a coordinate $P \mapsto P(d)$, or a negated coordinate. On operator encodings, Montague's individuals and their negations are the only nonconstant linear functionals of predicates, and every determiner over a restrictor of two or more elements, together with every modal or attitude operator over two or more indices, lies outside the linear regime. Corollary~\ref{cor:leibniz} says of $\varepsilon$ that it identifies two predicates when they agree on every individual, so that $\varepsilon$ is injective on predicates while its kernel as a linear map is the parallelogram space of Proposition~\ref{prop:cube}, of dimension $2^V - V - 1$.
\end{remark}

\subsection{Affine readout}\label{sec:affine}

We see the all-ones row in Theorem~\ref{thm:rank} through the shared truth basis: the second row of $\bM_P$ is $\ones - \bt_P$; a linear $L_P$ must produce it. If $L_P$ is permitted to be affine of the form $L_P x = A x + \bc_P$, then the requirement changes.

\begin{proposition}[affine lifts]\label{prop:affine}
    Affine exact lifts of all of $\Lex$ over $\bH$ exist if and only if $\rs(\bT) \subseteq \rs(\bH) + \operatorname{span}\{\ones\}$; equivalently, if and only if linear exact lifts exist over the homogenized geometry $\bH^{+} = [\bH; \ones] \in \R^{(d+1) \times V}$.
\end{proposition}

\begin{proof}[Proof sketch]\label{prf:affine}
    An affine map on $\R^d$ is a linear map on $\R^{d+1}$, restricted to the affine hyperplane of vectors with last coordinate $1$, and $h^{+}(d_i) = (h(d_i), 1)$ has geometry matrix $\bH^{+}$; the equivalence of affine lifts over $\bH$ and linear lifts over $\bH^{+}$ is this identification. Theorem~\ref{thm:rank}, applied to $\bH^{+}$, requires $\rs[\bT; \ones] \subseteq \rs(\bH^{+}) = \rs(\bH) + \operatorname{span}\{\ones\}$, in which the $\ones$ requirement is automatic, leaving $\rs(\bT) \subseteq \rs(\bH) + \operatorname{span}\{\ones\}$.
\end{proof}

\begin{corollary}[minimal affine dimension]\label{cor:affmindim}
    Write $\bT_c = \bT - \bar{\bt}\ones^{\top}$ for the truth matrix with its row means removed. Then $\rk \bT_c = \rk[\bT;\ones] - 1$, and the least $d$ for which some $\bH \in \R^{d \times V}$ carries affine exact lifts of all of $\Lex$ is $\rk[\bT;\ones] - 1$, attained by any $\bH$ whose rows form a basis of $\rs(\bT_c)$.
\end{corollary}

\begin{proof}\label{prf:affmindim}
    Each row of $\bT_c$ sums to zero, so $\rs(\bT_c) \subseteq \ones^{\perp}$ and $\ones \notin \rs(\bT_c)$, while $\rs[\bT;\ones] = \rs(\bT_c) + \operatorname{span}\{\ones\}$, giving the rank identity. By Proposition~\ref{prop:affine}, affine exact lifts exist if and only if $\rs(\bT) \subseteq \rs(\bH) + \operatorname{span}\{\ones\}$, equivalently $\rs(\bT_c) \subseteq \rs(\bH) + \operatorname{span}\{\ones\}$. The right side has dimension at most $d + 1$ and must contain $\rs(\bT_c) + \operatorname{span}\{\ones\}$ of dimension $\rk[\bT;\ones]$, whence $d \ge \rk[\bT;\ones] - 1$, with equality when $\rs(\bH) = \rs(\bT_c)$.
\end{proof}

Appending a constant coordinate converts affine readouts into linear ones; the bias supplies the constant row required by the shared truth basis. Over the free carrier, linear and affine readouts realize the same predicate extensions, since $\ones\in\rs(I_V)$. We retain the linear formulation as primary, and obtain affine readouts by homogenization. Once either readout returns the truth basis, the Boolean connective matrices act unchanged.

\subsection{Connectives and sentences}\label{subsec:connectivesandsentences}

Once every predicate in $\Lex$ lifts exactly, so does every sentence built from atomic predications by truth-functional connectives, with the same connective matrices as over the free carrier.

\begin{proposition}\label{prop:sentences}
    Suppose every $P \in \Lex$ has an exact lift $L_P$ over $\bH$. Then, for every Boolean combination $\Phi$ of atomic predications $P(d_i)$, the vector computed by applying $\bM_c$ to tensor products of lifted outputs equals $h_t(\den{\Phi})$.
\end{proposition}

\begin{proof}\label{prf:sentences}
    We proceed by induction on $\Phi$. For the simple case: $L_P h(d_i) = h_t(P(d_i))$ by hypothesis. Inductive case: if the immediate subformulas evaluate to $h_t(t_1), \dots, h_t(t_n)$, then $\bM_c(h_t(t_1) \otimes h_t(t_2)\otimes\cdots \otimes h_t(t_n)) = h_t(c(t_1, \dots, t_n))$ by the definition of $\bM_c$ over the truth space, which is unchanged by compression of the entity space.
\end{proof}

Compression of the entity carrier is, therefore, confined to the leaves of a derivation; the connective level of the vector logic is insensitive to it. This is a first indication of where the compressed regime places its constraints: on the geometry of atoms and their readouts; everything from the truth space upward is unchanged. The compound predicates themselves need no lifts of their own, and in general have none: given the $\ones$ row, the exact lexicon is closed under negation, since $\ones - \bt_P \in \rs(\bH)$, while $\bt_P \odot \bt_Q$, the truth row of $P \wedge Q$ as a monadic predicate, may lie outside $\rs(\bH)$, as it does for the two predicates of Example~\ref{ex:running} at $V = 4$, where the conjunction is the singleton $\{d_1\}$ that Figure~\ref{fig:bothregimes} excludes. Proposition~\ref{prop:sentences} evaluates $P(d_i) \wedge Q(d_i)$ exactly all the same, because the tensor product of the two readouts is quadratic in $h(d_i)$, and the connective matrix acts on that product.

\begin{example}\label{ex:running}
    Take $\mathcal{D}_e = \{\dan, \tom\}$, so $V = 2$, and the lexicon $\Lex = \{\mathrm{write}, \mathrm{published}\}$ with $\bt_{\mathrm{write}} = (1, 0)$ and $\bt_{\mathrm{published}} = (0, 0)$, so that at the world of evaluation Daniel writes and neither is published. Then $[\bT; \ones]$ has rows $(1,0), (0,0), (1,1)$ and rank $2 = V$; by Corollary~\ref{cor:mindim}, the minimal exact dimension is $2$, so this lexicon already forces the free regime on two entities. Extend to $V = 4$, with $\bt_{\mathrm{write}} = (1,1,0,0)$ and $\bt_{\mathrm{published}} = (1,0,1,0)$: now $\rk[\bT; \ones] = 3 < 4$, exact compression to $d = 3$ exists, and \Cref{sec:parallel} identifies the one admissible dependence.
\end{example}

\begin{figure}[ht]
    \centering
    \begin{tikzpicture}[x=1.5cm, y=1.3cm, >={Latex[length=1.8mm]}, baseline=(current bounding box.south)]
    \draw[black!50] (-0.35,1) -- (1.55,1);
    \node[right, font=\footnotesize, black!60] at (1.55,1) {height $1$};
    \draw[->, black!40] (0,0) -- (1.75,0) node[below, font=\footnotesize, black] {write};
    \draw[->, black!40] (0,0) -- (0,1.55);
    \draw[->, thick] (0,0) -- (1,1);
    \draw[->, thick] (0,0) -- (0,1);
    \fill (1,1) circle (1.6pt) node[above, font=\footnotesize] {$h(\dan)$};
    \fill (0,1) circle (1.6pt) node[above left, font=\footnotesize] {$h(\tom)$};
    \node[below, font=\footnotesize] at (0.85,-0.45) {free, $V = 2$};
    \end{tikzpicture}
    \hspace{3em}
    \begin{tikzpicture}[x={(2.0cm,0cm)}, y={(1.15cm,0.62cm)}, z={(0cm,1.7cm)}, >={Latex[length=1.8mm]}, baseline=(current bounding box.south)]
    \fill[black!6] (-0.18,-0.18,1) -- (1.30,-0.18,1) -- (1.30,1.30,1) -- (-0.18,1.30,1) -- cycle;
    \draw[black!50] (-0.18,-0.18,1) -- (1.30,-0.18,1) -- (1.30,1.30,1) -- (-0.18,1.30,1) -- cycle;
    \draw[->, black!40] (0,0,0) -- (1.75,0,0) node[below, font=\footnotesize, black] {write};
    \draw[->, black!40] (0,0,0) -- (0,1.8,0) node[right=1pt, font=\footnotesize, black] {published};
    \draw[->, black!40] (0,0,0) -- (0,0,1.5);
    \draw[->, thin, black!60] (0,0,0) -- (1,1,1);
    \draw[->, thin, black!60] (0,0,0) -- (1,0,1);
    \draw[->, thin, black!60] (0,0,0) -- (0,1,1);
    \draw[->, thin, black!60] (0,0,0) -- (0,0,1);
    \draw[thick] (0,0,1) -- (1,0,1) -- (1,1,1) -- (0,1,1) -- cycle;
    \foreach \p/\lab/\pos in {{(1,1,1)}/{$h(d_1)$}/{above right}, {(1,0,1)}/{$h(d_2)$}/{right}, {(0,1,1)}/{$h(d_3)$}/{above left}} {
      \fill \p circle (1.6pt);
      \node[\pos, font=\footnotesize, fill=white, inner sep=1pt] at \p {\lab};
    }
    \fill (0,0,1) circle (1.6pt);
    \node[left=3pt, font=\footnotesize, fill=white, inner sep=1pt] at (0,0,1) {$h(d_4)$};
    \node[below, font=\footnotesize] at (0.85,0.55,-0.62) {compressed, $V = 4$, $d = 3$};
    \end{tikzpicture}
    \caption{Both regimes on the lexicon of Example~\ref{ex:running}, in minimal exact geometries, whose rows are truth rows and $\ones$, the zero row of $\mathrm{published}$ at $V = 2$ dropped; each coordinate reads off a predicate (or the constant); every entity vector reaches the affine slice at height one. Left: $V = 2$: $\rk[\bT;\ones] = 2 = V$, the lexicon forces the free regime, and the two entity vectors are independent; the sole dependence $\sum_i c_i\, h(d_i) = 0$ is $c = 0$. Right: $V = 4$: $\rk[\bT;\ones] = 3$, and the four entity vectors satisfy the single admissible dependence $h(d_1) - h(d_2) - h(d_3) + h(d_4) = 0$, a parallelogram as in Section~\ref{sec:parallel}. Every truth row of the lexicon annihilates it, so $\bw_{\mathrm{write}} = (1,0,0)$ and $\bw_{\mathrm{published}} = (0,1,0)$ are exact; the singleton $\{d_1\}$ assigns it weight $1$, and (3) of Theorem~\ref{thm:rank} excludes its exact lift over this geometry.}
    \label{fig:bothregimes}
\end{figure}
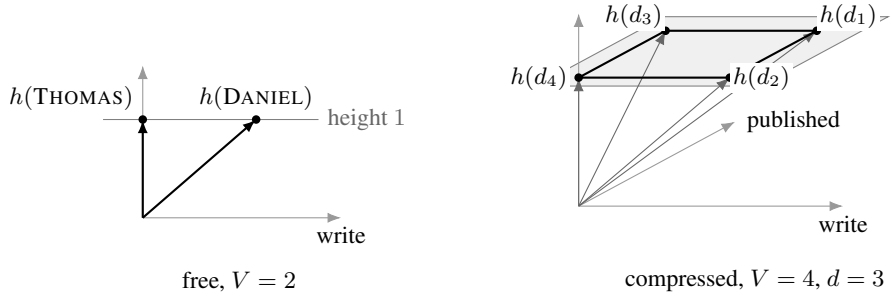

\subsection{Determiners}\label{subsec:determiners}

Once the atomic lexicon lifts exactly, quantification over entities computes in the compressed space as well. A determiner meaning that is conservative, extension-invariant, and isomorphism-invariant depends on its restrictor $P$ and scope $Q$ only through the pair $(|P \setminus Q|, |P \cap Q|)$ \cite{vanBenthem1986,KeenanStavi1986}, the tree of numbers, so that \textit{every} is $|P \setminus Q| = 0$, \textit{some} is $|P \cap Q| \ge 1$, \textit{most} is $|P \cap Q| > |P \setminus Q|$, and \textit{at least $n$} is $|P \cap Q| \ge n$.

\begin{proposition}[compressed determiners]\label{prop:determiners}
    Suppose every $P \in \Lex$ has an exact lift over $\bH$, with readouts $\bw_P \bH = \bt_P$ and $\bw_\ones \bH = \ones$, and let $\bG = \bH \bH^{\top} \in \R^{d \times d}$. Then, for all $P, Q \in \Lex$,
    \[
      |P \cap Q| = \bw_P\, \bG\, \bw_Q^{\top}, \qquad |P \setminus Q| = \bw_P\, \bG\, (\bw_\ones - \bw_Q)^{\top},
    \]
    so every conservative, extension-invariant, isomorphism-invariant determiner evaluates on $P$ and $Q$ by two $d \times d$ bilinear accumulations followed by its decision on the tree of numbers.
\end{proposition}

\begin{proof}\label{prf:determiners}
    $|P \cap Q| = \bt_P \bt_Q^{\top} = \bw_P \bH \bH^{\top} \bw_Q^{\top}$ and $|P \setminus Q| = \bt_P (\ones - \bt_Q)^{\top} = \bw_P \bH \bH^{\top} (\bw_\ones - \bw_Q)^{\top}$; the rest is the cited classification.
\end{proof}

The Gram matrix $\bG$ is formed once from the geometry and shared across the lexicon. The quantified sentence is a decision on bilinear forms in the readouts and, by the descent theorem of \cite{quigley2026intensional}, a linear readout of no single vector. The two accumulations have the form of the modal accumulation \eqref{eq:box} with the restrictor in place of the accessible set, which is the sense in which modals quantify over worlds \cite{Kratzer1991}; the compressed regime places one constraint on both, and leaves the connective and quantifier levels as they are. Beyond the exact regime, the counts inherit the defect of Section~\ref{sec:relax} through $\bt_P - \bw_P \bH$, and the decisions then act on approximate counts, which we leave to future work.

\section{Relations}\label{sec:rel}

A binary relation $R \subseteq \mathcal{D}_e \times \mathcal{D}_e $ is, at type $\langle e, \langle e, t \rangle \rangle$, a curried function, and the $\Hom$ construction of \cite{quigley2026intensional} lifts it to a linear map $\mathcal{S}_{\mathcal{D}_e} \to \Hom(\mathcal{S}_{\mathcal{D}_e}, \mathcal{S}_{\mathcal{D}_t})$, equivalently, a bilinear map $\mathcal{S}_{\mathcal{D}_e} \times \mathcal{S}_{\mathcal{D}_e} \to \mathcal{S}_{\mathcal{D}_t}$, equivalently, a linear map on the tensor square. Over a compressed geometry, the corresponding object is a linear $L_R : \R^d \otimes \R^d \to \R^2$ with
\begin{equation}\label{eq:rellift}
  L_R\, \bigl(h(d_i) \otimes h(d_j)\bigr) = h_t(R(d_i, d_j)) \quad \text{for all } i, j.
\end{equation}
Write $\bT_R \in \{0,1\}^{V \times V}$ for the truth matrix of $R$, $(\bT_R)_{ij} = R(d_i, d_j)$.

\subsection{Bilinear criterion}\label{subsec:bilinear}

\begin{theorem}[relations]\label{thm:rel}
    An exact lift \eqref{eq:rellift} of $R$ over $\bH$ exists if and only if $\ones \in \rs(\bH)$, and there is $\bA_R \in \R^{d \times d}$ with
    \[
      \bT_R = \bH^{\top} \bA_R\, \bH .
    \]
\end{theorem}

\begin{proof}[Proof sketch]\label{prf:rel}
    Bookkeeping is an irritant and laborious here. We proceed cautiously.

    The tensor square $h(d_i) \otimes h(d_j)$ is column $(i,j)$ of $\bH \otimes \bH \in \R^{d^2 \times V^2}$, so \eqref{eq:rellift} is $L_R (\bH \otimes \bH) = \bM_R$, with $\bM_R \in \R^{2 \times V^2}$ having columns $h_t(R(d_i, d_j))$ in the same ordering of pairs $(i,j)$ as the Kronecker product; as in Lemma~\ref{lem:rowspace}, $\rs(\bM_R) = \operatorname{span}\{\vecop(\bT_R)^{\top}, \ones_{V^2}^{\top}\}$ with $\vecop$ taken in that ordering, and solvability is $\rs(\bM_R) \subseteq \rs(\bH \otimes \bH)$. The rows of $\bH \otimes \bH$ are the Kronecker products $\bH_k \otimes \bH_l$ of pairs of rows of $\bH$, so $\rs(\bH \otimes \bH) = \rs(\bH) \otimes \rs(\bH)$; under the identification of $\R^{1 \times V} \otimes \R^{1 \times V}$ with $V \times V$ matrices, $\rs(\bH) \otimes \rs(\bH) = \{\bH^{\top} \bA \bH : \bA \in \R^{d \times d}\}$, since $\rs(\bH) = \{\mathbf{a}^{\top} \bH\}$ and $(\mathbf{a}^{\top} \bH)^{\top} (\mathbf{a}'^{\top} \bH) = \bH^{\top} (\mathbf{a} \mathbf{a}'^{\top}) \bH$, with sums of such rank-one terms filling out all $\bA$. So $\vecop(\bT_R)^{\top} \in \rs(\bH \otimes \bH)$ is $\bT_R = \bH^{\top} \bA_R \bH$ for some $\bA_R$, and $\ones_{V^2}^{\top} = \ones_V^{\top} \otimes \ones_V^{\top} \in \rs(\bH) \otimes \rs(\bH)$ if and only if $\ones_V^{\top} \in \rs(\bH)$, since $\ones \ones^{\top} = \bH^{\top} \bA \bH$ requires the rank-one matrix $\ones\ones^{\top}$ to have column space inside $\cs(\bH^{\top}) = \rs(\bH)^{\top}$, and conversely $\ones_V^{\top} = \mathbf{a}^{\top} \bH$ gives $\ones\ones^{\top} = \bH^{\top} \mathbf{a}\, \mathbf{a}^{\top} \bH$.
\end{proof}

\begin{figure}[ht]
    \centering
    \begin{tikzpicture}[x=1cm, y=1cm, >={Latex[length=1.6mm]}]
    \fill[black!12] (0,2.5) rectangle (3.4,2.72);
    \fill[black!12] (2.3,0) rectangle (2.52,3.4);
    \fill[black!75] (2.3,2.5) rectangle (2.52,2.72);
    \fill[black!25] (4.6,2.5) rectangle (5.4,2.72);
    \fill[black!25] (9.3,1.3) rectangle (9.52,2.1);
    \draw (0,0) rectangle (3.4,3.4);
    \draw (4.6,0) rectangle (5.4,3.4);
    \draw (5.8,1.3) rectangle (6.6,2.1);
    \draw (7.0,1.3) rectangle (10.4,2.1);
    \node at (1.7,1.7) {$\bT_R$};
    \node at (4.0,1.7) {$=$};
    \node at (5.0,1.7) {$\bH^{\top}$};
    \node at (6.2,1.7) {$\bA_R$};
    \node at (8.7,1.7) {$\bH$};
    \node[below, font=\footnotesize] at (1.7,0) {$V \times V$};
    \node[below, font=\footnotesize] at (5.0,0) {$V \times d$};
    \node[below, font=\footnotesize] at (6.2,1.3) {$d \times d$};
    \node[below, font=\footnotesize] at (8.7,1.3) {$d \times V$};
    \node[above, font=\footnotesize] (ent) at (2.41,3.4) {$R(d_i,d_j)$};
    \draw[->, thin] (ent.south) -- (2.41,2.76);
    \node[above, font=\footnotesize] (hi) at (5.0,3.4) {$h(d_i)^{\top}$};
    \draw[->, thin] (hi.south) -- (5.0,2.76);
    \node[below, font=\footnotesize] (hj) at (9.41,0.75) {$h(d_j)$};
    \draw[->, thin] (hj.north) -- (9.41,1.26);
    \end{tikzpicture}
    \caption{Factorization $\bT_R = \bH^{\top} \bA_R\, \bH$, with $d \ll V$. The shaded row of $\bH^{\top}$ is $h(d_i)^{\top}$ and the shaded column of $\bH$ is $h(d_j)$; the light bands in $\bT_R$ are row $i$ and column $j$, and their crossing is the entry $(\bT_R)_{ij} = R(d_i, d_j)$, the readout $h(d_i)^{\top} \bA_R\, h(d_j)$. The inner dimension is $d$, giving $\rk \bT_R \le d$, the obstruction of Corollary~\ref{cor:relrank}; the second condition of the theorem, $\ones \in \rs(\bH)$, constrains $\bH$ alone.}
    \label{fig:rel}
\end{figure}
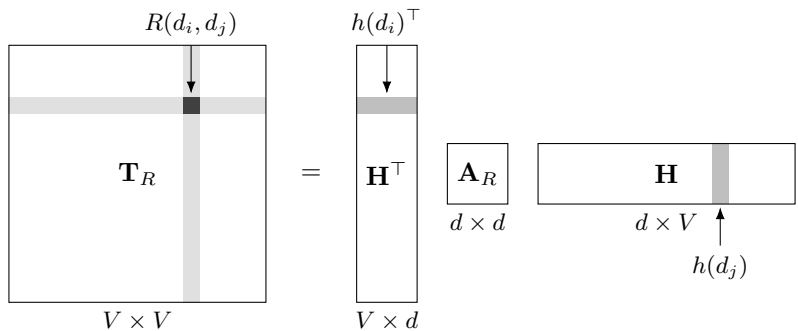

The lifted relation\footnote{This is the bilinear scoring function of relational embedding models such as RESCAL \cite{nickel2011three}, which Theorem~\ref{thm:rel} recovers as the exact form of a lifted binary relation; under the vector logic, the score has a truth-conditional reading.} is a bilinear form $\bA_R$ on the compressed space, and $R(d_i, d_j)$ is read off as $h(d_i)^{\top} \bA_R\, h(d_j)$. 

\subsection{Obstructions}\label{subsec:obstruction}

\begin{corollary}[rank obstruction]\label{cor:relrank}
    If $R$ lifts exactly over $\bH$, then $\rk \bT_R \le \rk \bH \le d$. In particular, the identity relation $\{(d_i, d_i)\}$, with $\bT_{=} = I_V$, lifts exactly only over free geometries.
\end{corollary}

\begin{proof}\label{prf:relrank}
    $\rk(\bH^{\top} \bA_R \bH) \le \rk \bH$; and $\rk I_V = V$ forces $\rk \bH = V$.
\end{proof}

The rank bound is a lower bound only; the factorization constrains both argument positions through the same geometry, and the exact value follows from the proof of Theorem~\ref{thm:rel}.

\begin{corollary}[minimal dimension for a relation]\label{cor:relmin}
    The least $d$ for which some $\bH \in \R^{d \times V}$ carries an exact lift of $R$ is $\dim \operatorname{span}\bigl(\{\ones\} \cup \rs(\bT_R) \cup \rs(\bT_R^{\top})\bigr)$, attained by any $\bH$ whose rows form a basis of that span.
\end{corollary}

\begin{proof}\label{prf:relmin}
    By the identification in the proof of Theorem~\ref{thm:rel}, $\{\bH^{\top} \bA \bH : \bA \in \R^{d \times d}\}$ is the set of $V \times V$ matrices whose rows lie in $\rs(\bH)$ and whose columns lie in $\rs(\bH)^{\top}$, so $\bT_R = \bH^{\top} \bA_R \bH$ is solvable if and only if $\rs(\bT_R) \subseteq \rs(\bH)$ and $\rs(\bT_R^{\top}) \subseteq \rs(\bH)$. With the requirement $\ones \in \rs(\bH)$, exact lift is containment of the stated span in $\rs(\bH)$, whence $d \ge \rk \bH \ge \dim \operatorname{span}\bigl(\{\ones\} \cup \rs(\bT_R) \cup \rs(\bT_R^{\top})\bigr)$, with equality when the rows of $\bH$ are a basis of the span.
\end{proof}

A single monadic predicate has $\rk[\bt_P; \ones] \le 2$, and lifts exactly in dimension two, so forcing high dimension at type $\langle e, t \rangle$ requires a lexicon, and Corollary~\ref{cor:closure} uses the whole closed such one. A single binary relation can force the free regime by itself, and the relation that does so is identity, the denotation of the copula in \textit{Daniel is Daniel}. A strict total order does so as well, one dimension above its rank: $\bT_<$ is strictly upper triangular with ones above the diagonal, of rank $V - 1$, while $\rs(\bT_<) = \operatorname{span}\{\be_2^{\top}, \dots, \be_V^{\top}\}$ and $\rs(\bT_<^{\top}) = \operatorname{span}\{\be_1^{\top}, \dots, \be_{V-1}^{\top}\}$ together span $\R^{1 \times V}$ for $V \ge 2$, so Corollary~\ref{cor:relmin} returns $V$. Equivalence relations with $k$ classes have $\rk \bT_R = k$, and, since $\bT_R$ is symmetric with $\ones$ in its row space, compress to dimension $k$ exactly. The rank of a relation's truth matrix bounds the dimension of any exact carrier from below, in the same way that $\rk[\bT; \ones]$ bounds the dimension for a monadic lexicon, and Corollary~\ref{cor:relmin} gives the value.

\begin{remark}[higher arities]
    An $n$-ary relation imposes its condition on $\bH^{\otimes n}$, with $\rk(\bH^{\otimes n}) = (\rk \bH)^n$, and the lift is an $n$-linear form; the flattening ranks of the truth tensor bound $\rk \bH$ from below in the same way. Theorem~\ref{thm:rel} is the $n = 2$ case.
\end{remark}

\section{Index sorts}\label{sec:index}

The intensional layer places its free carrier on the index space, $h_S(s) = \be_s$; for a discrete sort with finitely many indices, the results of Sections~\ref{sec:rank} and \ref{sec:rel} apply to any compressed geometry $\bH_W \in \R^{d \times n}$ of the world sort with the same proofs\footnote{We defer them here, and encourage the ambitious reader to try them.}, once the objects are identified. Here, propositions play the role of monadic predicates: $\varphi$ has truth profile $\bv(\varphi)^{\top}$ as its truth row over $W$, and a lexicon of propositions $\Lex_W$ has truth matrix $\bT_W$. Accessibility plays the role of a binary relation, with truth matrix $\bA$.

\begin{proposition}[compressed worlds]\label{prop:worlds}
    Let $\bH_W \in \R^{d \times n}$ be a geometry of $W$.
    \begin{enumerate}
      \item Every $\varphi \in \Lex_W$ has an exact lift $L_\varphi \bH_W = \bP_\varphi$ if and only if $\rs(\bH_W) \supseteq \rs[\bT_W; \ones]$; the minimal exact dimension is $\rk[\bT_W; \ones]$.
      \item Accessibility lifts exactly as a bilinear form if and only if $\ones \in \rs(\bH_W)$ and $\bA = \bH_W^{\top} \widehat{\bA} \bH_W$ for some $\widehat{\bA} \in \R^{d \times d}$; in particular, $\rk \bA \le d$.
      \item Under (1) and (2) the modal accumulation of \eqref{eq:box} computes in the compressed space: with $\bv(\varphi) = (\bw_\varphi \bH_W)^{\top}$ and $\ones = (\bw_\ones \bH_W)^{\top}$ for the readouts of (1),
      \[
        \bA\, \bv(\varphi) = \bH_W^{\top}\, \widehat{\bA}\, \bH_W \bH_W^{\top}\, \bw_\varphi^{\top}, \qquad \bA\, (\ones - \bv(\varphi)) = \bH_W^{\top}\, \widehat{\bA}\, \bH_W \bH_W^{\top}\, (\bw_\ones - \bw_\varphi)^{\top},
      \]
      each a $d \times d$ computation followed by one expansion through $\bH_W^{\top}$, after which the decisions of \eqref{eq:box} apply unchanged.
    \end{enumerate}
\end{proposition}

\begin{proof}
    (1) and (2) are simply Theorems~\ref{thm:rank} and \ref{thm:rel}, with $W$ for $\mathcal{D}_e$. (3) is substitution.
\end{proof}

\begin{figure}[ht]
    \centering
    \begin{tikzpicture}[x=1cm, y=1cm]
    \draw (0,0) rectangle (3.4,3.4);
    \draw (3.75,0) rectangle (4.0,3.4);
    \draw (4.9,0) rectangle (5.7,3.4);
    \draw (6.1,1.3) rectangle (6.9,2.1);
    \draw (7.3,1.3) rectangle (8.1,2.1);
    \draw (8.6,1.3) rectangle (8.85,2.1);
    \node at (1.7,1.7) {$\bA$};
    \node[above, font=\footnotesize] at (3.875,3.4) {$\bv(\varphi)$};
    \node at (4.45,1.7) {$=$};
    \node at (5.3,1.7) {$\bH_W^{\top}$};
    \node at (6.5,1.7) {$\widehat{\bA}$};
    \node[above, font=\footnotesize] at (7.7,2.1) {$\bH_W^{} \bH_W^{\top}$};
    \node[above, font=\footnotesize] at (8.725,2.1) {$\bw_\varphi^{\top}$};
    \node[below, font=\footnotesize] at (1.7,0) {$n \times n$};
    \node[below, font=\footnotesize] at (3.875,0) {$n \times 1$};
    \node[below, font=\footnotesize] at (5.3,0) {$n \times d$};
    \node[below, font=\footnotesize] at (6.5,1.3) {$d \times d$};
    \node[below, font=\footnotesize] at (7.7,1.3) {$d \times d$};
    \node[below, font=\footnotesize] at (8.725,1.3) {$d \times 1$};
    \draw[decorate, decoration={brace, mirror, amplitude=5pt}] (4.9,-0.55) -- (5.7,-0.55) node[midway, below=6pt, font=\footnotesize, align=center] {one\\expansion};
    \draw[decorate, decoration={brace, mirror, amplitude=5pt}] (6.1,-0.55) -- (8.85,-0.55) node[midway, below=6pt, font=\footnotesize] {$d \times d$ computation};
    \end{tikzpicture}
    \caption{The compressed modal accumulation of (3), $\bA\, \bv(\varphi) = \bH_W^{\top}\, \widehat{\bA}\, (\bH_W \bH_W^{\top})\, \bw_\varphi^{\top}$, with $d \ll n$. To the right of $\bH_W^{\top}$, every block is $d \times d$ or $d \times 1$, so the accumulation is a $d \times d$ computation, and one expansion through $\bH_W^{\top}$; the threshold checks of \eqref{eq:box} then apply to the expanded vector unchanged. The $\bH_W \bH_W^{\top}$ is as a single $d \times d$ block, because it is formed once from the geometry and shared across propositions $\varphi$.}
    \label{fig:worlds}
\end{figure}
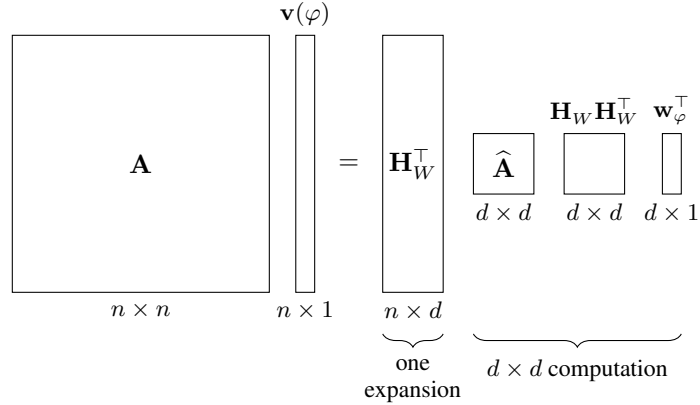

The reflexive and transitive frames of the modal logics that concern the intensional layer have accessibility matrices of varying rank; the empty relation has rank zero, and imposes only the $\ones$ requirement, a universal relation has rank one and compresses to a single dimension, a strict linear order, of rank $n - 1$, forces the free regime by Corollary~\ref{cor:relmin}, its row and column spaces together spanning $\R^{1 \times n}$, as does a reflexive linear order, whose matrix is upper triangular with unit diagonal, and the identity relation (the frame of the trivial modality), by Corollary~\ref{cor:relrank}; the strict future accessibility of a discrete time sort, therefore, admits exact lift only over the free carrier of that sort. Since the operators of \eqref{eq:box} read $\bA$ only through its Boolean support, an approximate carrier need only reproduce the support of $\bA$ for the modal verdicts to survive, a weaker requirement than exact bilinear factorization\footnote{The product-modality is future work, which concerns the support factorization in its own right; we do not pursue the approximate modal case here.}.

\section{Relaxation}\label{sec:relax}

Exact lift is a subspace condition that either holds or fails outright; learned geometries will fail it for most predicates, so we might wonder, then how far is a geometry from carrying a predicate at all? The direct construction in the proof of Theorem~\ref{thm:rank} suggests that exact lift places $\bt_P$ in $\rs(\bH)$, and the distance from $\bt_P$ to $\rs(\bH)$ is a least-squares quantity.

\subsection{Defect}\label{subsec:relax}

\begin{definition}[defect]\label{def:defect}
    For a predicate $P$ with truth row $\bt_P$ and a geometry $\bH$, the defect is
    \[
      \delta(P; \bH) = \min_{\bw \in \R^{1 \times d}} \bigl\lVert \bw \bH - \bt_P \bigr\rVert_2 .
    \]
\end{definition}

\begin{figure}[ht]
    \centering
    \begin{tikzpicture}[x=1cm, y=1cm, >={Latex[length=2mm]}]
    \coordinate (A) at (0,0);
    \coordinate (B) at (6,0);
    \coordinate (C) at (7.4,1.6);
    \coordinate (D) at (1.4,1.6);
    \fill[black!6] (A) -- (B) -- (C) -- (D) -- cycle;
    \draw[black!50] (A) -- (B) -- (C) -- (D) -- cycle;
    \node[anchor=north west, font=\footnotesize] at (5.2,0.02) {$\rs(\bH)$};
    \coordinate (O) at (2.2,0.55);
    \coordinate (T) at (4.6,3.1);
    \coordinate (P) at (4.6,0.95);
    \fill (O) circle (1.2pt);
    \draw[->, thick] (O) -- (T) node[above, font=\footnotesize] {$\bt_P$};
    \draw[->, thick] (O) -- (P) node[below=3pt, font=\footnotesize] {$\bt_P \bH^{\dagger}\bH$};
    \draw[dashed] (P) -- (T);
    \node[right=1pt, font=\footnotesize] at ($(P)!0.55!(T)$) {$\delta(P;\bH)$};
    \draw ($(P)+(-0.18,0)$) -- ($(P)+(-0.18,0.18)$) -- ($(P)+(0,0.18)$);
    \draw ($(O)!0.55cm!(P)$) arc[start angle=9.5, end angle=46.7, radius=0.55cm];
    \node[font=\footnotesize] at ($(O)+(0.85,0.32)$) {$\theta$};
    \draw[->, thick, black!60] (O) -- (0.9,1.25) node[left, font=\footnotesize] {$\ones$};
    \end{tikzpicture}
    \caption{The defect of Definition~\ref{def:defect}, in which the truth row $\bt_P$ is projected onto $\rs(\bH)$; the length is $\delta(P;\bH)$ and the angle $\theta$ is the principal angle between $\bt_P$ and the row space, with $\delta = \lVert \bt_P \rVert \sin\theta$. Exact lift is $\theta = 0$. The $\ones$ row is drawn nearly in the plane, as Section~\ref{sec:measure} finds it for trained geometries; the principal angles of Table~\ref{tab:measure} are the angles $\theta$ taken jointly over $\rs[\bT;\ones]$.}
    \label{fig:defect}
\end{figure}
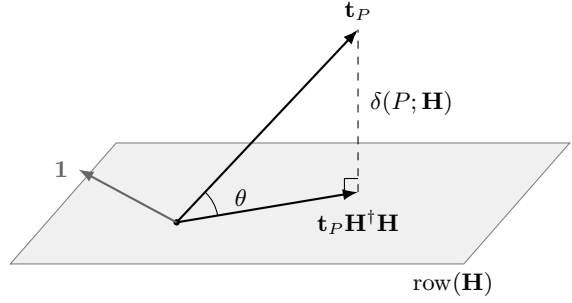

\begin{proposition}\label{prop:defect}
    $\delta(P; \bH) = \lVert \bt_P (I_V - \bH^{\dagger} \bH) \rVert_2$, where $\bH^{\dagger}$ is the Moore--Penrose pseudoinverse, and $\bH^{\dagger} \bH$ is the orthogonal projector onto $\rs(\bH)$; the minimizer is $\bw = \bt_P \bH^{\dagger}$. Moreover, $\delta(P; \bH) = 0$ if and only if $\bt_P \in \rs(\bH)$, so that if $\ones \in \rs(\bH)$, exact lift of $P$ is $\delta(P; \bH) = 0$.
\end{proposition}

\begin{proof}[Proof sketch]\label{prf:defect}
    $\{\bw \bH\}$ is $\rs(\bH)$, and the closest point of a subspace to $\bt_P$ is its orthogonal projection $\bt_P \bH^{\dagger} \bH$, with $\bt_P (I - \bH^{\dagger} \bH)$ vanishes exactly on the subspace. The final clause is Theorem~\ref{thm:rank}.
\end{proof}

The defect is computable on any trained embedding by one least-squares solve per predicate, with predicate extensions supplied by lexical resources or feature norms \cite{mcrae2005semantic}; it is a statistic of the trained geometry and  a property of the training data and objective, which is unconstrained; the condition is, therefore, a specification, and the defect is a measurement.

\subsection{Threshold and separability}\label{subsec:thresholdseparability}

Conceptual spaces \cite{Gardenfors2000} and degree semantics \cite{kennedy2007vagueness} treat graded predicates as regions in a structured space, with the classical predicate recovered by a threshold, $\den{P}(x) = \mathbf{1}[\, d(x, R_P) \le \theta_P \,]$; its simplest instance is a half-space, what the exact lift becomes when the equality in \eqref{eq:lift} is weakened to a sign condition.

\begin{definition}[thresholded lift]\label{def:thresholdlift}
    $P$ is linearly separable over $\bH$ if there are $\bw \in \R^{1 \times d}$ and $\theta \in \R$ with $\bw h(d_i) > \theta$ when $P(d_i) = 1$ and $\bw h(d_i) < \theta$ when $P(d_i) = 0$.
\end{definition}

\begin{proposition}\label{prop:sep}
    If $P$ has an exact lift over $\bH$, then $P$ is linearly separable over $\bH$; the converse fails.
\end{proposition}

\begin{proof}\label{prf:sep}
    With $\bw = \bw_P$ from the direct construction, $\bw h(d_i) = P(d_i) \in \{0, 1\}$, and $\theta = \tfrac12$ separates. For the converse, take $d = 1$, $\bH = [\,1\;2\;3\;4\,]$, and $\bt_P = (0,0,1,1)$: the threshold $\theta = \tfrac52$ separates, while $\bt_P \in \operatorname{span}\{\bH, \ones\}$ would require $a + c = 0$ and $2a + c = 0$ from the first two coordinates, forcing $a = c = 0$, which contradicts the third coordinate; exact lift, therefore, fails, in the affine sense of Proposition~\ref{prop:affine}, as well as the linear one.
\end{proof}

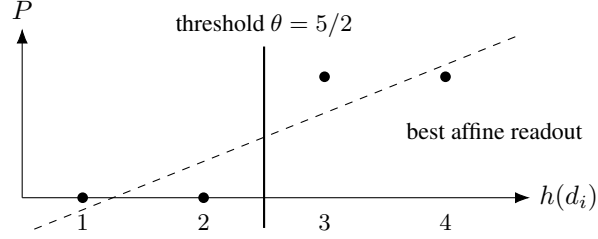
\begin{figure}[ht]
    \centering
    \begin{tikzpicture}[x=1.6cm, y=1.6cm]
    \draw[-{Latex[length=2mm]}] (0.5,0) -- (4.7,0) node[right] {$h(d_i)$};
    \draw[-{Latex[length=2mm]}] (0.5,0) -- (0.5,1.4) node[above] {$P$};
    \foreach \x/\y in {1/0, 2/0, 3/1, 4/1} \fill (\x,\y) circle (2pt);
    \foreach \x in {1,2,3,4} \node[below=3pt, font=\footnotesize] at (\x,0) {$\x$};
    \draw[dashed] (0.6,-0.5+0.4*0.6) -- (4.6,-0.5+0.4*4.6);
    \node[font=\footnotesize, anchor=west] at (3.6,0.55) {best affine readout};
    \draw[thick] (2.5,-0.25) -- (2.5,1.25);
    \node[font=\footnotesize, above] at (2.5,1.25) {threshold $\theta = 5/2$};
    \end{tikzpicture}
    \caption{The example of \Cref{prop:sep}: four entities on a line, $\bt_P = (0,0,1,1)$. No affine function of position takes the values $0,0,1,1$, so the affine defect is positive, while the threshold at $5/2$ separates the extension exactly. The measurements of Section~\ref{sec:measure} find learned geometries in this position for nearly every predicate: high separability, positive defect.}
    \label{fig:threshold}
\end{figure}

A linear probe tests the linear separability of a predicate over a learned geometry: \cite{alain2017understanding} keep the probe linear, so that accuracy tracks the representation, and \cite{hewitt2019designing} caution that it tracks the capacity of the probe as well; we must account for this, so we do so with held-out scoring throughout, and a Gaussian null on the held-out error of Section~\ref{sec:measure}. Proposition~\ref{prop:sep} is the probe in the vector logic: thresholded relaxation of the homomorphism condition, strictly weaker than the exact condition, and the defect $\delta(P; \bH)$ is the condition's own measure. Probing practice fits each predicate independently; the rank criterion adds that a shared readout basis across the lexicon requires the $\ones$-row, or, equivalently, a bias, and that the ranks of Corollaries~\ref{cor:mindim} and \ref{cor:relrank} bound what any geometry of a given dimension can carry, before any probe is fit.

\subsection{Parallelograms}\label{sec:parallel}

Let us now return to the four-entity lexicon of Example~\ref{ex:running}: $\bt_{\mathrm{write}} = (1,1,0,0)$, $\bt_{\mathrm{published}} = (1,0,1,0)$, $\ones = (1,1,1,1)$. The three rows are independent, so $\rk[\bT; \ones] = 3$, and, by Corollary~\ref{cor:dual}, the admissible dependences form a one-dimensional space (the orthogonal complement of the row space), spanned by
\[
  \bc = (1, -1, -1, 1)^{\top}:
  \qquad
  \bt_{\mathrm{write}}\, \bc = 1 - 1 = 0, \quad
  \bt_{\mathrm{published}}\, \bc = 1 - 1 = 0, \quad
  \ones\, \bc = 0 .
\]
Every minimal exact geometry for this lexicon, therefore, satisfies exactly one dependence,
\[
  h(d_1) - h(d_2) = h(d_3) - h(d_4),
\]
which form a parallelogram, with the interpretation: the difference between a writer who is published and one who is unpublished equals the difference between a nonwriter who is published and one who is unpublished. The simplest nontrivial solutions of the constraints the rank criterion imposes are analogy structures of the kind reported for word embeddings since the classic \cite{mikolov2013linguistic}. 

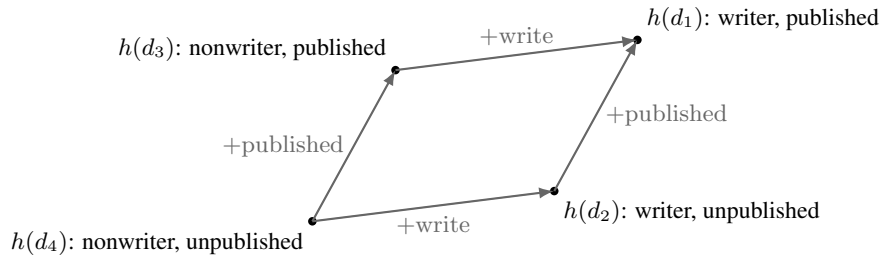
\begin{figure}[ht]
    \centering
    \begin{tikzpicture}[x=1cm, y=1cm, >={Latex[length=2mm]}]
    \coordinate (h4) at (0,0);
    \coordinate (h2) at (3.2,0.4);
    \coordinate (h3) at (1.1,2.0);
    \coordinate (h1) at (4.3,2.4);
    \draw[black!40] (h4) -- (h2) -- (h1) -- (h3) -- cycle;
    \foreach \p/\lab/\pos in {h1/{$h(d_1)$: writer, published}/above right, h2/{$h(d_2)$: writer, unpublished}/below right, h3/{$h(d_3)$: nonwriter, published}/above left, h4/{$h(d_4)$: nonwriter, unpublished}/below left} {
      \fill (\p) circle (1.6pt);
      \node[\pos, font=\footnotesize] at (\p) {\lab};
    }
    \draw[->, thick, black!60] (h2) -- (h1) node[midway, right, font=\footnotesize] {$+\mathrm{published}$};
    \draw[->, thick, black!60] (h4) -- (h3) node[midway, left, font=\footnotesize] {$+\mathrm{published}$};
    \draw[->, thick, black!60] (h4) -- (h2) node[midway, below, font=\footnotesize] {$+\mathrm{write}$};
    \draw[->, thick, black!60] (h3) -- (h1) node[midway, above, font=\footnotesize] {$+\mathrm{write}$};
    \end{tikzpicture}
    \caption{The parallelogram forced on every minimal exact geometry for the four-entity lexicon (simplified from Figure~\ref{fig:bothregimes}): $h(d_1) - h(d_2) = h(d_3) - h(d_4)$, so the displacement for \textit{published} is the same, whether taken from a writer or from a nonwriter, and likewise for \textit{write}. The single admissible dependence $\bc = (1,-1,-1,1)$ is this figure.}
    \label{fig:parallelogram}
\end{figure}

\begin{remark}[A note on parallelogram]
    Corollary~\ref{cor:dual} gives a parallelogram kernel to every minimal exact geometry for a two-feature lexicon on four entities, an exact geometry of higher rank having kernel zero, and the empirical literature says trained geometries (approximately) have parallelogram kernels for feature pairs of this shape; the two routes to the parallelogram are distinct: the present derivation applies to exact geometries, and the measured geometries of Section~\ref{sec:measure} are inexact, so the analogies observed in trained embeddings are accounted for from the training objective, through the co-occurrence statistics it factorizes \cite{ethayarajh2019analogies,allen2019analogies}.
\end{remark}

In general, the admissible dependences are $\rs[\bT; \ones]^{\perp}$, and integer vectors in that complement with two entries $+1$ and two entries $-1$ are exactly the parallelograms the lexicon permits at all. This is not deep, but follows from pairs of entity pairs that agree on feature difference.

\begin{proposition}[Boolean cube]\label{prop:cube}
    Let $\Lex$ consist of $m$ binary features on the $V = 2^m$ entities of $\{0,1\}^m$, feature $k$ having truth row $x \mapsto x_k$. Then $\rk[\bT; \ones] = m + 1$, and the space of admissible dependences $\rs[\bT; \ones]^{\perp}$, of dimension $2^m - m - 1$, is spanned by the parallelogram vectors $\be_x - \be_{x + \be_k} - \be_z + \be_{z + \be_k}$ over coordinates $k$ and points $x, z$ with $x_k = z_k = 0$.
\end{proposition}

\begin{proof}\label{prf:cube}
    The rows $x \mapsto x_k$ and $x \mapsto 1$ are the affine functions' basis on the cube, and are independent, giving the rank. Each parallelogram vector is orthogonal to every affine function $f(x) = a_0 + \sum_k a_k x_k$, since $f(x) - f(x + \be_k) - f(z) + f(z + \be_k) = -a_k + a_k = 0$, so the parallelogram span lies in $\rs[\bT; \ones]^{\perp}$. 
    
    \noindent Conversely, let $f \in \R^V$ be orthogonal to every parallelogram vector; then $f(x + \be_k) - f(x) = f(z + \be_k) - f(z)$ for all $x, z$ with $x_k = z_k = 0$, so the increment along coordinate $k$ is a constant $a_k$, and induction on the number of nonzero coordinates gives $f(x) = f(0) + \sum_k a_k x_k$, an affine function. The orthogonal complement of the parallelogram span is, therefore, the space of affine functions, and the parallelogram span is its complement.
\end{proof}

\begin{figure}[ht]
    \centering
    \begin{tikzpicture}[x={(1cm,0cm)}, y={(0.32cm,1cm)}, z={(0.62cm,0.36cm)}, >={Latex[length=2mm]}]
    \coordinate (000) at (0,0,0);
    \coordinate (100) at (2.8,0,0);
    \coordinate (010) at (0,2.2,0);
    \coordinate (110) at (2.8,2.2,0);
    \coordinate (001) at (0,0,2.0);
    \coordinate (101) at (2.8,0,2.0);
    \coordinate (011) at (0,2.2,2.0);
    \coordinate (111) at (2.8,2.2,2.0);
    \fill[black!10] (010) -- (110) -- (111) -- (011) -- cycle;
    \draw[black!45, dashed] (000) -- (001);
    \draw[black!45, dashed] (001) -- (101);
    \draw[black!45, dashed] (001) -- (011);
    \draw[black!60] (000) -- (100) -- (110) -- (010) -- cycle;
    \draw[black!60] (100) -- (101) -- (111) -- (110);
    \draw[black!60] (010) -- (011) -- (111);
    \draw[thick] (010) -- (110) -- (111) -- (011) -- cycle;
    \foreach \p in {000,100,010,110,001,101,011,111} \fill (\p) circle (1.4pt);
    \node[below left, font=\footnotesize] at (000) {$000$};
    \node[below right, font=\footnotesize] at (100) {$100$};
    \node[left, font=\footnotesize] at (010) {$010$};
    \node[below right=-2pt, font=\footnotesize] at (110) {$110$};
    \node[above left=-2pt, font=\footnotesize] at (001) {$001$};
    \node[right, font=\footnotesize] at (101) {$101$};
    \node[above left, font=\footnotesize] at (011) {$011$};
    \node[above right, font=\footnotesize] at (111) {$111$};
    \draw[->, thick] (000) -- ($(000)!0.5!(100)$) node[midway, below, font=\footnotesize] {$a_1$};
    \draw[->, thick] (000) -- ($(000)!0.5!(010)$) node[midway, left, font=\footnotesize] {$a_2$};
    \draw[->, thick] (000) -- ($(000)!0.5!(001)$) node[midway, above right=3pt, font=\footnotesize] {$a_3$};
    \node[font=\footnotesize, align=left, anchor=west] at (5.4,1.4) {$h(x) = a_0 + \sum_k a_k x_k$\\[2pt] shaded face: $\be_{010} - \be_{110} - \be_{011} + \be_{111}$\\[2pt] $\rk[\bT;\ones] = 4$, $\;\dim \rs[\bT;\ones]^{\perp} = 4$};
    \end{tikzpicture}
    \caption{The Boolean cube of Proposition~\ref{prop:cube} for $m = 3$: eight entities, three features, and a minimal exact geometry in $\R^4$ (drawn in three dimensions, with the affine offset suppressed). Every minimal exact geometry is an affine image of the cube, so each face is a parallelogram; the four independent faces span the admissible dependences, of dimension $2^3 - 3 - 1 = 4$, and the affine functions $a_0 + \sum_k a_k x_k$ are licensed by the lexicon.}
    \label{fig:cube}
\end{figure}
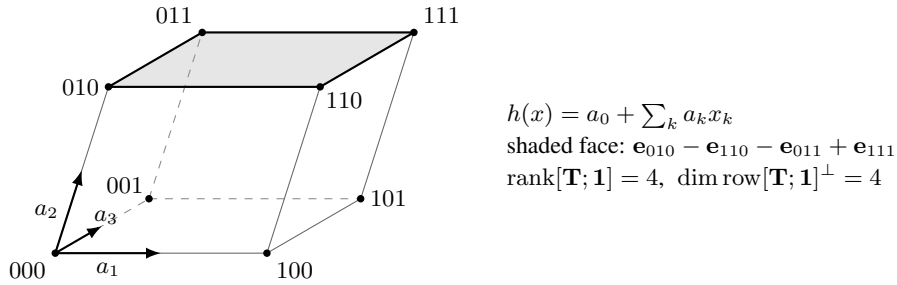

In a minimal exact geometry for a full factorial feature lexicon, every coordinate difference is, thus, a constant vector, which is the setting in which analogy by vector arithmetic is exact; an exact geometry of higher rank satisfies a subset of these equalities, and the free geometry keeps the $2^m$ entity vectors affinely independent.

\subsection{Measurement}\label{sec:measure}

We turn, now, to implementation and experiment\footnote{Code is available at the project repo.
Minimal working verifiable code was written by the author; LLM assistance for this paper was used in the writing and expanding of the experiment software at larger scale, building from the author's minimal working code.}. We computed a held-out error $\hat{\delta}$, written so as to keep it apart from the defect $\delta$ of Definition~\ref{def:defect}, the probe, the principal angles, and the dimension sweep of Corollary~\ref{cor:mindim}, on two classic lexicons against two likewise classic embeddings. The geometries are a 300-dimensional GloVe embedding \cite{pennington2014glove} and the 300-dimensional word2vec embedding of \cite{mikolov2013distributed}; each is used as its geometry matrix $\bH$ without centering, whitening, normalization, or truncation. The projection residual of Proposition~\ref{prop:defect} depends on $\rs(\bH)$ alone; the held-out error below depends on the coordinates through its penalty, and is reported on the coordinates as published. 

The first lexicon uses the McRae feature norms \cite{mcrae2005semantic}. Following the norms' own inclusion threshold, a feature holds of a concept when at least five participants produced it. Entries below the threshold enter $\bT$ as false, so the truth-conditional reading treats nonproduction as a negative judgment, which is an assumption about the norms rather than a datum in them.

Concepts the norms distinguish by sense, such as \textit{bat} in its animal and baseball senses, share one word vector and receive the union of their features. The 541 concepts collapse to 532 words, all present in GloVe and 531 in word2vec. Retaining features assigned to at least fifteen words, and to at most fifteen fewer than all of them, gives 76 predicates and $\rk[\bT;\ones]=77$ in both embeddings.

The second lexicon uses WordNet hypernyms \cite{fellbaum1998wordnet}. We select monosemous nouns among each embedding's twenty thousand most frequent tokens. Hypernyms of depth at least four with 30--500 members become predicates. GloVe supplies 6271 nouns and 171 predicates, with 165 distinct extensions and $\rk[\bT;\ones]=166$; for word2vec the figures are 5257 nouns and 143 predicates, with 139 distinct extensions and $\rk[\bT;\ones]=139$. Predicates with identical extensions remain separate rows.

In every case, $\rk[\bT;\ones]<d=300<V$, which places all four cells in the compressed regime. The available dimension therefore permits exact monadic compression, although whether a pretrained geometry realizes it remains to be tested.

The full-domain quantity is the affine defect of Proposition~\ref{prop:affine}, normalized by the centered norm of the truth row,
\[
  \delta^{+}(P; \bH) = \frac{\bigl\lVert \bt_P \bigl(I_V - (\bH^{+})^{\dagger} \bH^{+}\bigr) \bigr\rVert_2}{\lVert \bt_P - \bar{t}_P \ones \rVert_2}, \qquad \bar{t}_P = |P| / V,
\]
so that a constant readout scores $1$ and an exact affine lift scores $0$. The held-out error $\hat{\delta}(P; \bH)$ is a predictive quantity, and differs from it in target, in regularization, and in normalization. The entities are split into five folds $F_1, \dots, F_5$; for each $k$, a readout $(\bw^{(k)}, b^{(k)})$ is fit by ridge regression over $\bH^{+}$ on the entities outside $F_k$, with the intercept unpenalized, and
\[
  \hat{\delta}(P; \bH) = \left( \frac{\sum_{k} \sum_{i \in F_k} \bigl(\bw^{(k)} h(d_i) + b^{(k)} - \bt_P(i)\bigr)^2}{\sum_{k} \sum_{i \in F_k} \bigl(\bt_P(i) - \bar{t}^{(k)}_P\bigr)^2} \right)^{1/2},
\]
with $\bar{t}^{(k)}_P$ the mean of $\bt_P$ over $F_k$, so predicting the fold mean scores $1$. The penalty is chosen per predicate from $\{10^{-2}, 10^{-1}, 1, 10, 10^{2}, 10^{3}\}$ by the minimum of this same held-out error, a selection that biases $\hat{\delta}$ downward; choosing it on inner folds of the training data instead moves the medians below by at most $0.01$. The null is the same statistic on a Gaussian matrix of the same shape. A positive $\hat{\delta}$ is, therefore, consistent with an exact affine lift on the full domain, which a penalized fit to a subset need not recover, and the two quantities are reported side by side. The probe is a logistic regression with unit inverse regularization and balanced class weights, its probabilities cross-fitted over stratified five-fold splits and scored by the area under the curve. Ranks are taken at the default tolerance of the numerical library, the largest singular value times the larger matrix dimension times machine precision; principal angles are the arccosines of the singular values of $Q_U^{\top} Q_W$, with $Q_U$ and $Q_W$ orthonormal bases of $\rs(\bH)$ and $\rs[\bT;\ones]$ from singular value decompositions at a relative cutoff of $10^{-10}$. Table~\ref{tab:measure} gives medians over predicates.

\begin{table}[h]
    \centering
    \begin{tabular}{llrrlrrrr}
    \toprule
    \textbf{Lexicon} & \textbf{Geometry} & \textbf{\boldmath $V$} & \textbf{\boldmath $\rk[\bT;\ones]$} & \textbf{\boldmath $\delta^{+}$ (min)} & \textbf{\boldmath $\hat{\delta}$} & \textbf{\boldmath $\hat{\delta}$ null} & \textbf{AUC} & \textbf{Angle band} \\
    \midrule
    McRae   & GloVe    &  532 &  77 & 0.51 (0.25) & 0.85 & 1.03 & 0.95 & 10--18 (20--26) \\
    McRae   & word2vec &  531 &  77 & 0.53 (0.25) & 0.89 & 1.03 & 0.96 & 10--20 (21--27) \\
    WordNet & GloVe    & 6271 & 166 & 0.89 (0.66) & 0.93 & 1.02 & 0.95 & 32--45 (68--70) \\
    WordNet & word2vec & 5257 & 139 & 0.86 (0.54) & 0.92 & 1.02 & 0.96 & 25--44 (67--69) \\
    \bottomrule
    \end{tabular}
    \caption{Median full-domain affine defect $\delta^{+}$, with its minimum over predicates in parentheses; median held-out error $\hat{\delta}$, with Gaussian null; median probe AUC; and the second through tenth principal angles in degrees between $\rs(\bH)$ and $\rs[\bT;\ones]$ (null given in parentheses). The first principal angle is $1.1$, $2.1$, $2.1$, and $5.2$ degrees; the angle between the $\ones$ row itself and $\rs(\bH)$ is $1.3$, $2.6$, $2.1$, and $5.4$ degrees, so the first principal direction lies near the $\ones$ row.}
    \label{tab:measure}
\end{table}

No predicate admits an exact affine lift in any of the four cells: the median full-domain defect ranges from $0.51$ to $0.89$ (see Table~\ref{tab:measure}), and the smallest over all predicates is $0.25$. The embeddings, nevertheless, have lower defects than Gaussian controls. For a centered truth row and a random $d$-dimensional subspace of $\ones^\perp$, the expected squared defect is $1-d/(V-1)$, and the observed control medians match its square root to two decimals.

Jointly over the lexicon, the principal angles give the same account of the obstruction: the lexical row space lies closer to $\rs(\bH)$ than the Gaussian controls do, and the constant row is nearly contained in it, while the smallest principal angle stays positive in every cell, beyond numerical tolerance. Writing
\[
  U=\rs(\bH), \qquad S=\rs[\bT;\ones],
\]
we, therefore, have
\[
  U\cap S=\{0\},
  \qquad
  \bigl(U+\operatorname{span}\{\ones\}\bigr)\cap S
  =\operatorname{span}\{\ones\}.
\]
The second equality follows because $\ones\in S$: if $u+c\ones\in S$ with $u\in U$, then $u\in U\cap S$. Thus no predicate of the lexicon admits an exact linear readout, and none other than a constant one an exact affine readout, agreeing with the individual defects.

Under the held-out error, recovery is imperfect as well: only one McRae predicate, \textit{musical instrument}, has $\hat{\delta}<0.5$, and none does on WordNet. Held-out error and probe AUC are strongly negatively correlated on McRae ($\rho=-0.86$ and $-0.82$), so predicates with better discrimination generally have lower prediction error. High AUC, however, establishes neither exact affine recovery nor strict separability.

We test separability directly by seeking $\bw$ and $b$ such that
\[
  \bw h(d_i)+b \ge 1 \quad\text{if } P(d_i)=1,
  \qquad
  \bw h(d_i)+b \le -1 \quad\text{otherwise}.
\]
Every McRae predicate is separable in both embeddings and their Gaussian controls, so this test does not distinguish the geometries there. On WordNet, GloVe separates $159$ of $171$ predicates and word2vec $136$ of $143$, compared with $87$ and $80$ in the controls.

Size accounts for the controls, which separate no predicate above $63$ members in the GloVe cell or $61$ in the word2vec cell. The geometries separate every predicate up to $100$ members, $19$ of $22$ and $21$ of $24$ between $100$ and $200$, and $3$ of $12$ and $3$ of $7$ above; among the predicates larger than any sampled control managed, $62$ of $74$ over GloVe and $53$ of $60$ over word2vec are half-spaces. Failures include \textit{action}, \textit{activity}, and \textit{content}, whereas \textit{city}, \textit{municipality}, and \textit{urban area} are separable in both embeddings. \textit{Urban area} gives a direct instance of Proposition~\ref{prop:sep}, strictly separable at affine defects of $0.68$ and $0.54$; \textit{district} and \textit{administrative district} give the converse, the two predicates with the lowest held-out error over GloVe being among its failures.

Replacing first-sense WordNet labels with monosemous assignments raises median AUC from $0.89$ to $0.95$, and lowers held-out error only from $0.94$ to $0.93$. Retaining progressively more principal directions of $\bH$ lowers held-out error gradually, without a pronounced transition at the lexicon's rank, which Corollary~\ref{cor:mindim} permits: the bound guarantees that some geometry of sufficient dimension carries the lexicon, not that a truncation of a pretrained geometry does.

Recovery varies by predicate type as well: on McRae, taxonomic predicates have median held-out errors of $0.65$ and $0.71$, against $0.87$ and $0.90$ for attributive predicates, with color and size worst; on WordNet, administrative and geographic categories and substance nouns are predicted best, and abstract nouns such as \textit{idea} and \textit{information} worst.\footnote{This agrees with \cite{rubinstein2015well}, who found better distributional recovery of taxonomic than attributive properties.}

\begin{figure}[h]
    \centering
    \includegraphics[width=0.75\textwidth]{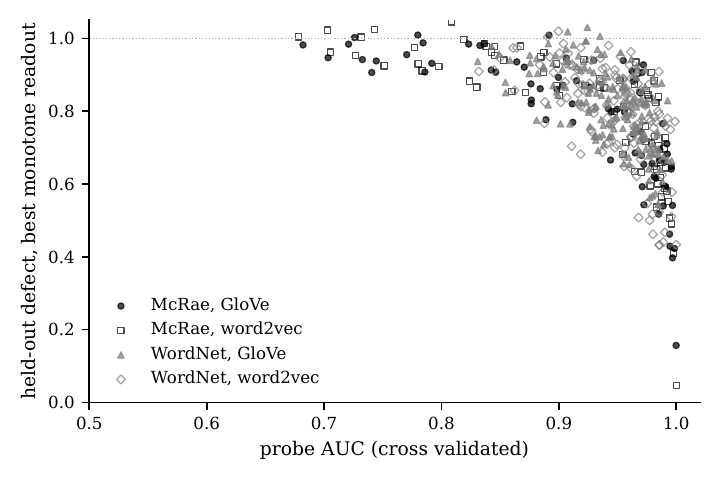}
    \caption{Per-predicate held-out error $\hat{\delta}$ under the isotonic readout against cross-validated probe AUC, drawn from the per-predicate tables. Predicates concentrate at high AUC and high error: discriminable and poorly recovered. The lower right corner, low error at high AUC, is nearly empty.}
    \label{fig:defectauc}
\end{figure}

Monotone transformations of the affine score lower the held-out error to between $0.77$ and $0.86$, against Gaussian controls near $1.00$, a gain over the ridge readout of $0.03$ to $0.06$ on McRae and $0.10$ to $0.15$ on WordNet. Both are fit within training folds, isotonic regression on cross-fitted training scores.

Recovery is much better in a small tail: for \textit{musical instrument}, on eighteen concepts, isotonic error falls to $0.16$ and $0.05$, and birds and their parts, fruit, cities, and countries follow between $0.40$ and $0.60$, at AUC above $0.98$. The tenth percentile of the isotonic error lies between $0.54$ and $0.68$, so substantial error remains for most predicates.

A two-layer readout with 64 hidden units (validated on synthetic data requiring nonlinear recovery) raises median error on McRae by $0.03$ and $0.04$ relative to isotonic regression, and lowers it on WordNet by $0.06$ and $0.07$, to $0.77$ and $0.75$, while its AUC does not exceed the logistic probe's. These are results about held-out entities: on the finite domain itself, distinct entity vectors permit recovery of every predicate by an unrestricted decoder.

No predicate of either lexicon admits an exact linear or affine lift, so the geometries lie outside the exact regime of Theorem~\ref{thm:rank}. Many predicates are, nevertheless, strictly separable, placing the WordNet cells inside the thresholded regime of Proposition~\ref{prop:sep} well beyond the sampled controls, and the McRae cells inside it at a rate the controls match. Held-out recovery varies across predicates and remains imperfect under every readout family examined. Exact representability, separability, and predictive performance, therefore, give distinct assessments of one geometry.

\subsection{Training toward exact regime}\label{sec:train}

Whether the exact regime is reachable at the dimensions current embeddings use, and at what distributional cost, is the question of this section.

We train a geometry $\bH \in \R^{d \times V}$ under a mixed objective. Let $\bE \in \R^{300 \times V}$ be the pretrained geometry with its row means removed, and $\bT_c$ the truth matrix with its row means removed. For readouts $\bW \in \R^{300 \times d}$ and $\bW_T \in \R^{|\Lex| \times d}$ and an intercept $\bw_0 \in \R^{|\Lex|}$,
\begin{equation}\label{eq:mixed}
  \Loss(\bH) = (1 - \lambda)\, \min_{\bW} \frac{\lVert \bW \bH - \bE \rVert_F^2}{\lVert \bE \rVert_F^2}
  + \lambda\, \min_{\bW_T, \bw_0} \frac{\lVert \bW_T \bH + \bw_0 \ones^{\top} - \bT \rVert_F^2}{\lVert \bT_c \rVert_F^2},
\end{equation}
so the distributional term is the squared relative error of reconstructing the pretrained vectors linearly from $\bH$, and the truth term is the squared relative residual of the affine lift of Proposition~\ref{prop:affine}, taken jointly over the lexicon. The objective is minimized in these squared terms; the numbers reported below, and plotted in Figures~\ref{fig:switch} and~\ref{fig:frontier}, are their square roots, written $\epsilon(\bH)$ for the distributional reconstruction error and $\delta_{\Lex}(\bH)$ for the joint truth defect, so that both are relative norms on the scale of $\delta^{+}$. Both inner minima are least-squares problems with closed-form solutions, and $\Loss$ is minimized by alternating least squares between $\bH$ and the readouts, with a ridge of $10^{-8}$ on the $\bH$ solve and the rows of $\bH$ renormalized after each step, since $\Loss$ is invariant under $GL(d)$ acting on $\bH$; each run starts from a Gaussian $\bH$ with a fixed seed and takes forty iterations. Every entity is a column of $\bH$, so the experiment is transductive, as an embedding layer is. At $\lambda = 0$, training reconstructs the pretrained geometry; at $\lambda = 1$, the truth term alone is minimized, the surplus directions of $\bH$ above the rank are untrained, and the distributional error at $\lambda = 1$ carries no information, so the frontier is read at $\lambda \in (0, 1)$.

Two quantities are available in closed form. 
\begin{enumerate}
    \item The minimum of the truth term over all $d$-dimensional geometries is $\bigl(\sum_{k > d} \sigma_k^2\bigr)^{1/2} / \lVert \bT_c \rVert_F$ for the singular values $\sigma_k$ of $\bT_c$, since the intercept absorbs the row means; this floor is positive for $d < \rk \bT_c = \rk[\bT;\ones] - 1$ by Corollary~\ref{cor:affmindim}, and zero from there on, which is Corollary~\ref{cor:mindim} in the affine form of Proposition~\ref{prop:affine}, and it fixes the location of the exactness transition in advance. 

    \item The second is a linear-exact benchmark: the variance retained at dimension $d$ under the constraint $\rs(\bH) \supseteq \rs[\bT;\ones]$ of Theorem~\ref{thm:rank} is the fraction of $\lVert \bE \rVert_F^2$ captured by the best $d$-dimensional row space containing $\rs[\bT;\ones]$, namely $\rs[\bT;\ones]$ together with the top $d - \rk[\bT;\ones]$ principal directions of $\bE$ projected off it, and at a least-squares optimum the retained variance is $1 - \epsilon^2$.
\end{enumerate}

 The benchmark is one dimension more constrained than the training criterion, which is affine, and needs $\rs(\bH)$ to contain a complement of $\ones$ in $\rs[\bT;\ones]$, such as $\rs(\bT_c)$; the affine benchmark, with $\rs(\bT_c)$ in place of $\rs[\bT;\ones]$, retains at least as much, and the two differ by at most the variance of one direction.

The grid runs over $d \in \{10, 20, 40, 80, 120, 160, 200, 300\}$ and $\lambda \in \{0, 0.1, 0.5, 0.9, 1\}$. A predicate counts as recovered within tolerance when its affine defect on the trained geometry, $\lVert \bw_P \bH + b_P \ones^{\top} - \bt_P \rVert_2 / \lVert \bt_P - \bar{t}_P \ones \rVert_2$ with the readout refit by least squares, is below $0.05$; Figure~\ref{fig:switch} plots this fraction. The McRae cell with word2vec in this was matched to the embedding by the training loader, which looks tokens up as given, whereas the diagnostics loader of Section~\ref{sec:measure} adds a case fallback; that holds $V = 529$, one predicate falls below fifteen positives, and it trains on $75$ predicates with $\rk[\bT;\ones] = 76$, while Table~\ref{tab:measure} reports the diagnostics cell with $76$ predicates and rank $77$; the dotted line of Figure~\ref{fig:switch} for that cell sits at $75$ accordingly.

At $\lambda = 1$, the optimizer attains the closed-form floor of $\delta_{\Lex}$ at every grid point: to four decimals where the floor is positive (on McRae $0.702$, $0.554$, and $0.350$ at $d = 10, 20, 40$; on WordNet with GloVe $0.014$ at $d = 160$; with word2vec $0.058$ at $d = 120$), and to $10^{-13}$ where the floor is zero, which is every grid point at or above $\rk[\bT;\ones] - 1$. The affine exact regime is, therefore, reached at numerical tolerance at $d = 80$ and above on McRae, at $d = 200$ and above on WordNet with GloVe, and at $d = 160$ and above with word2vec, and the location of the transition is the floor's, $\rk[\bT;\ones] - 1$, with the grid serving to check that the optimizer finds it. The per-predicate fraction within tolerance at $\lambda = 1$ is $1$ at those points, $0.947$ at $d = 160 < 165$ on WordNet with GloVe, and $0.706$ at $d = 120 < 138$ with word2vec, which is the approximate carriage a positive floor admits: below the rank no geometry carries the whole lexicon, and most of it can still lie within tolerance.

At $d = 300$, the constraint is relatively cheap on McRae and relatively affordable on WordNet: the linear-exact benchmark retains $99.2$ and $98.5$ percent of the pretrained variance on McRae and $82.5$ and $79.5$ percent on WordNet, the constraint having rank $166$ and $139$ over a spectral tail at $V \approx 6000$. On the frontier at $\lambda = 0.9$, $\delta_{\Lex}$ is $0.003$ and $0.004$ on McRae at $\epsilon = 0.085$ and $0.121$ (retention $99.3$ and $98.5$ percent), with every predicate within tolerance; on WordNet it is $0.046$ and $0.044$ at $\epsilon = 0.365$ and $0.415$ (retention $86.7$ and $82.8$ percent), with $0.75$ of the predicates within tolerance after forty iterations.

The affine exact regime is, therefore, reached at numerical tolerance at $d = 80$ and above on McRae, at $d = 200$ and above on WordNet with GloVe, and at $d = 160$ and above with word2vec, which are the grid points at or above the minimal affine dimension of Corollary~\ref{cor:affmindim}.

\begin{figure}[ht]
    \centering
    \includegraphics[width=0.85\textwidth]{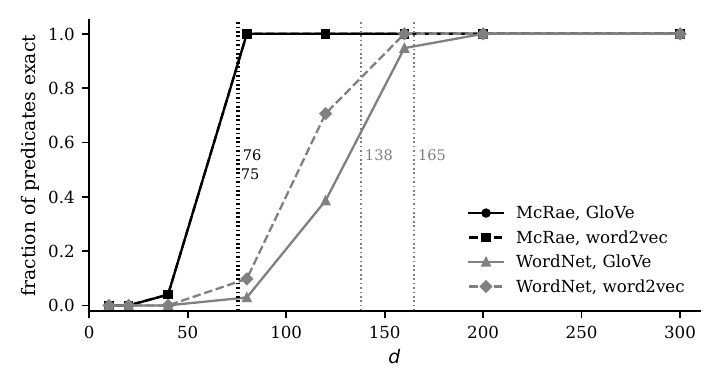}
    \caption{Fraction of predicates with affine defect below $0.05$ on the trained geometry (recovered within tolerance), after alternating least squares at $\lambda = 1$, against the trained dimension $d \in \{10, 20, 40, 80, 120, 160, 200, 300\}$, drawn from the run's output. Dotted lines mark the minimal affine dimension $\rk[\bT;\ones] - 1$ of Corollary~\ref{cor:affmindim}: $76$ and $75$ (McRae; the word2vec training cell holds $75$ predicates, see the text), $165$ (WordNet, GloVe), $138$ (WordNet, word2vec). Wherever the fraction reads $1$, the joint residual is below $10^{-13}$; at the two grid points just below that dimension it equals the closed-form minimum, $0.014$ and $0.058$.}
    \label{fig:switch}
\end{figure}

\begin{figure}[ht]
    \centering
    \includegraphics[width=0.85\textwidth]{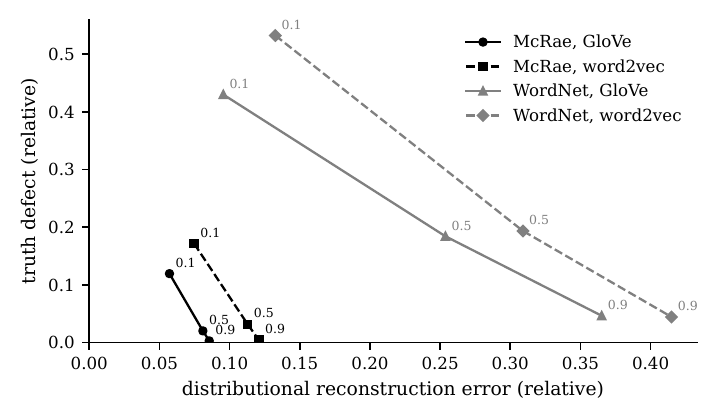}
    \caption{The frontier at $d = 300$: joint truth defect $\delta_{\Lex}$ against distributional reconstruction error $\epsilon$, both relative, for $\lambda \in \{0.1, 0.5, 0.9\}$ after forty iterations, all four cells. On McRae, $\delta_{\Lex}$ reaches $0.003$ and $0.004$ at $\epsilon = 0.085$ and $0.121$; on WordNet, $0.046$ and $0.044$ at $0.365$ and $0.415$, where the optimizer is still descending at $\lambda = 0.9$.}
    \label{fig:frontier}
\end{figure}

Reaching the floor on the training entities shows that an embedding layer attains the affine exact regime under the truth objective; the frontier shows that most of the distributional variance survives within tolerance of it; exact realizability at $d \ge \rk[\bT;\ones]$ is Corollary~\ref{cor:mindim}, and the experiment shows that the optimizer finds it. The result is transductive: generalization of the truth-conditional structure to unseen entities, and acquisition of the same geometry from distributional training alone, remain open.

\section{Discussion}\label{sec:disc}

Consider the distinction between \emph{representability} and \emph{representation}. The free construction establishes when a truth-conditional lexicon can be carried exactly by a finite-dimensional vector space; what follows is the empirical analysis, in which we ask how closely existing representations approach that construction. The defect introduced above makes this quantitative: zero defect means exact, while positive defect measures difference from it.

We showed that ordinary distributional embeddings already lie substantially closer to the truth-conditional geometry than a random subspace of the same dimension; this is consistent with the fact that semantic features are often linearly represented in learned vector spaces \cite{park2024linear}; it also gives a geometric interpretation of superposition \cite{elhage2022toy}: when the available dimension is below the minimum required for exactness, several truth conditions must share dimensions, and the resulting dependencies appear as nonzero defect. The principal angles (as reported in Table~\ref{tab:measure}), therefore, quantify the extent to which a distributional geometry already contains the structure required by the lexicon, as opposed to mere generic similarity.

Distributional learning, by itself, leaves every predicate of both lexicons outside the exact regime, by the affine defects and the principal angles, and the held-out error orders the predicates: concrete category predicates are recovered far better than attributive and abstract predicates, and monotone (or nonlinear) readouts recover much of the same ordering. This is compatible with the literature on conceptual spaces \cite{Gardenfors2000}, while also placing a limit on a purely geometric account: a region in a conceptual space need not constitute an exact extension. Human categorization provides an independent reason not to expect exact linear separability as a universal property \cite{medin1981linear}.

At $d=300$, the truth-conditional constraint can be imposed, while preserving most of the variance of the original embeddings. On McRae, the lexicon is brought within tolerance with very little distributional loss; on WordNet, the constraint is more expensive, and remains compatible with substantial retention. Under the truth term alone, the optimizer reaches the affine exact regime at numerical tolerance at every trained dimension from $\rk[\bT;\ones] - 1$ upward, where the closed-form floor is zero, and matches the positive floor below it. Dimension is thereby eliminated as the reason the pretrained embeddings fail to carry the lexicon; the objective, the corpus, and the labels remain as candidates, which the present experiments leave apart from one another.

This has a useful consequence for the relation between distributional and truth-conditional semantics: the two need not compete for representation space. A single geometry can retain substantial distributional structure, while also carrying a truth-conditional lexicon. The experiments leave open whether language exposure alone produces such a geometry. Skip-gram's relation to shifted PMI factorization \cite{levy2014neural} gives the distributional geometry a corpus-level interpretation, but there is nothing in that objective that requires the resulting space to satisfy the truth-conditional constraints. The present results, therefore, support a weaker (and more precise) claim: \emph{distributional learning supplies information from which truth-conditional structure may be (partially) recovered; exactness requires either an additional constraint or some mechanism that supplies equivalent information.}

The framework also clarifies the status of compositionality. Once the leaves of a derivation are represented exactly, the homomorphism conditions determine the corresponding Boolean composition without further learning. Outside the exact regime, each leaf is only approximately represented, so compositional error can accumulate. A system may, therefore, perform well on individual semantic probes, while failing a composed entailment. Good distributional similarity alone, consequently, leaves exact logical behavior undetermined.

The empirical scope of the present study is deliberately narrower than the formal framework we are likewise developing; here, the measurements concern one-place predicates over static entity geometries. Relations, represented bilinearly on $\bH\otimes\bH$, and higher arities follow the same geometric strategy, but were not evaluated here, which we leave for future work. Quantifiers are treated in Proposition~\ref{prop:determiners} for the exact regime, where they compute by bilinear accumulation on the readouts, and their behavior under positive defect remains to be measured. Likewise, the formal conditions governing computation after the embedding layer raise a separate question. Proposition~\ref{prop:sentences} confines the present compression result to the leaves of a derivation; whether attention and feed-forward computation can, themselves, realize the required multilinear maps remains open.

The vector logic supplies a specification of what a representation carrying a truth-conditional structure has to satisfy, and the defect turns the specification into a measurable property of an empirical geometry, so that the question whether a learned representation carries such a structure has a computable answer. That is the principal role of the framework; the origin of the gradient observed here, and the emergence of truth-conditional structure from language exposure alone, lie outside what it decides.

\section{Conclusion}\label{sec:conclusion}

A truth-conditional lexicon imposes linear constraints on the representation space; the rank of those constraints gives the minimum dimension for exactness, while the defect measures how closely a lower-dimensional (or otherwise unconstrained) representation approaches that ideal. The empirical results show that existing distributional embeddings contain substantial structure relevant to the lexicon, while every predicate of both lexicons fails the exact criterion, linear and affine, over both geometries. Once the truth-conditional constraint enters the objective, the affine exact regime is attained at numerical tolerance from the predicted dimension $\rk[\bT;\ones] - 1$ upward, and, at $d = 300$, the lexicon comes within tolerance of it while most of the original distributional variance is retained; distributional training, by itself, leaves exactness to a further constraint.

The formal framework extends beyond the monadic case studied here. Relations and higher-arity predicates can be treated on tensor-product spaces, quantifiers require corresponding operators on compressed readouts, and contextual representations can be evaluated layer by layer. The most direct empirical continuation is to replace the pretrained distributional matrix in the mixed objective of Section~\ref{sec:train} with corpus statistics to test whether exact truth-conditional structure can emerge from co-occurrence information alone, and at what distributional cost; we expect not.

What does it mean for a vector geometry to carry a truth-conditional semantics compositionally is answered here for finite monadic lexicons and binary relations; how far does an empirical geometry fall short of that condition is measured for two lexicons and two embeddings; how a learning system might acquire such a representation is another problem.

\section*{Acknowledgments}\label{sec:acknowledgments}

The observation that an embedding layer is a linear map on one-hot inputs, and its pedagogical framing, are due entirely to Sebastian Raschka, which set this paper in motion.

\bibliography{compression}

\end{document}